\documentclass[10pt,journal]{IEEEtran}

\usepackage{cite}
\usepackage[pdftex]{graphicx}
\usepackage{amsfonts}
\usepackage{amsmath}
\usepackage{amssymb}
\usepackage{algorithm}
\usepackage{algpseudocode}

\usepackage{array}
\usepackage[caption=false,font=footnotesize,labelfont=sf,textfont=sf]{subfig}
\usepackage{url}
\usepackage{color}
\usepackage{hyperref}
\usepackage{booktabs}
\usepackage{ulem}
\usepackage[font=small,labelfont=bf]{caption}
\newtheorem{assumption}{Assumption}
\newtheorem{proposition}{Proposition}
\newtheorem{proof}{Proof}
\newtheorem{lemma}{Lemma}

\newcommand{\new}{\textcolor{black}}

\newcommand{\NAME}{\textsc{FedGS}}
\newcommand{\SAMPLERNAME}{GBP-CS}

\begin{document}

\title{Data Heterogeneity-Robust Federated Learning via Group Client Selection in Industrial IoT}

\author{
\IEEEauthorblockN{
Zonghang~Li,
Yihong~He,
Hongfang~Yu,~\IEEEmembership{Member~IEEE},
Jiawen~Kang,
Xiaoping~Li, 
Zenglin~Xu,~\IEEEmembership{Senior~Member~IEEE},
Dusit~Niyato,~\IEEEmembership{Fellow~IEEE}}
}

\IEEEtitleabstractindextext{
\begin{abstract}
\new{Nowadays, the industrial Internet of Things (IIoT) has played an integral role in Industry 4.0 and produced massive amounts of data for industrial intelligence. These data locate on decentralized devices in modern factories. To protect the confidentiality of industrial data,  federated learning (FL) was introduced to collaboratively train shared machine learning models. However, the local data collected by different devices skew in class distribution and degrade industrial FL performance. This challenge has been widely studied at the mobile edge, but they ignored the rapidly changing streaming data and clustering nature of factory devices, and more seriously, they may threaten data security. In this paper, we propose \NAME, which is a hierarchical cloud-edge-end FL framework for 5G empowered industries, to improve industrial FL performance on non-i.i.d. data. Taking advantage of naturally clustered factory devices, \NAME~uses a gradient-based binary permutation algorithm (\SAMPLERNAME) to select a subset of devices within each factory and build homogeneous super nodes participating in FL training. Then, we propose a compound-step synchronization protocol to coordinate the training process within and among these super nodes, which shows great robustness against data heterogeneity. The proposed methods are time-efficient and can adapt to dynamic environments, without exposing confidential industrial data in risky manipulation. We prove that \NAME~has better convergence performance than FedAvg and give a relaxed condition under which \NAME~is more communication-efficient. Extensive experiments show that \NAME~improves accuracy by 3.5\% and reduces training rounds by 59\% on average, confirming its superior effectiveness and efficiency on non-i.i.d. data.}
\end{abstract}
}

\maketitle
\IEEEdisplaynontitleabstractindextext
\IEEEpeerreviewmaketitle

\footnotetext[1]{
This work was supported in part by National Key Research and Development Program of China (2019YFB1802800), PCL Future Greater-Bay Area Network Facilities for Large-Scale Experiments and Applications (LZC0019).}

\footnotetext[2]{
Zonghang Li, Yihong He and Hongfang Yu are with School of Information and Communication Engineering, University of Electronic Science and Technology of China. (Emails: lizhuestc, heyh.uestc, yuhfnetworklab@gmail.com). Xiaoping Li is with School of Mathematical Science, University of Electronic Science and Technology of China. (Email: lixiaoping.math@uestc.edu.cn). Zenglin Xu is with School of Computer Science and Technology, Harbin Institute of Technology (Shenzhen), China. (Email: xuzenglin@hit.edu.cn). Dusit Niyato and Jiawen Kang is with School of Computer Science and Engineering, Nanyang Technological University, Singapore. (Email: dniyato, kavinkang@ntu.edu.sg). Corresponding author: Hongfang Yu.}

\footnotetext[3]{
This work was done in part at Nanyang Technological University (NTU).}

\section{Introduction}
In recent years, the Internet of Things (IoT) has played an increasingly integral role in the industrial community. \new{Taking logistics sorting and automatic object identification as examples. Optical character recognition (OCR) cameras on logistics pipelines detect and read characters on packing boxes in order to sort them\cite{chen2018high}. At the same time, the surrounding surveillance cameras are constantly monitoring, automatically identifying objects through optical recognition of the characters on their badges, and confirming whether the machines, robots, vehicles, and workers in the factory are legal entrants\cite{liukkonen2016toward}.} These optical sensors collect a huge amount of industrial data. In order to tap the value of these data, advanced data mining technologies are needed, especially machine learning (ML). However, gathering the industrial big data to the cloud leads to unbearable transmission overhead, and also violates data privacy regulations. Taking the idea of task offloading, federated learning (FL)\cite{mcmahan2021advances} sinks model training from the cloud to the edge. \new{OCR cameras use local optical data to train local OCR models, then upload their local model updates to the cloud to update the global model. The global model is then synchronized to OCR cameras. These steps are repeated until the global model converges.} In this way, FL preserves data confidentiality because the raw data does not leave the devices.

The combination of industrial IoT (IIoT) and FL opens a door for smart industry\cite{hiessl2020industrial,lim2020federated,parimala2021fusion}. FL provides powerful privacy-preserving tools for mining decentralized industrial data, and IIoT technologies such as smart sensors and mobile robots provide rich resources (e.g., data, computation) for FL. \new{Despite these benefits, compared with OCR in natural scenes, FL in industries requires higher accuracy to ensure the reliability of industrial operations. However, sensors' local data distributions can be highly heterogeneous due to differences in times, locations, functions, and so on. Taking the logistics industry of cross-border e-commerce as an example, the Singapore warehouse transports more packing boxes to Singapore than other countries, so the optical characters in the word ``Singapore" appear more times than other characters. For this reason, the number of optical images of each character captured by different OCR cameras (in different warehouses) is skewed and inconsistent. Other examples are device failure detection\cite{zhang2020blockchain} and object detection\cite{luo2019real}, which also prove the existence of data heterogeneity in real-world IIoT.} These skewed distributed data are called non-independent and identically distributed (non-i.i.d.) and can lead to FL performance degradation\cite{zhao2018federated}, \new{which becomes more challenging when local data of sensors are constantly changing.} 

\new{The non-i.i.d. data challenge has inspired the research field of heterogeneous FL, especially the field of mobile edge computing (MEC)\cite{lim2020federated,yao2019federated,zhang2021client,zhao2021federated,yoshida2019hybridfl,duan2019astraea,wang2021non,wen2020unified,sattler2020clustered,wang2020optimizing,zeng2022heterogeneous}, which currently remains open. These kinds of literature have achieved great success in the context of MEC, but the following characteristics of IIoT make them still limited:
\begin{enumerate}
\item \textbf{Higher requirements for data security.} Industries (e.g., manufacturing, logistics, and transportation) often face even more serious data threats due to owning a vast amount of valuable information, so they possess the most urgent and critical requirement to increase security to protect data. Therefore, any form of disclosure\cite{yao2019federated,yoshida2019hybridfl,zhang2021client,zhao2021federated} or tampering\cite{duan2019astraea,wang2021non,wen2020unified} of the confidential raw data is not allowed.
\item \textbf{Rapidly changing streaming data on data-intensive sensors.} Data-intensive IIoT sensors such as OCR cameras require high sampling rates to capture the real-time phenomenon information and produce large amounts of data. In order to save storage space, these data will overwrite the old data that has been processed, forming a data stream similar to a first-in-first-out (FIFO) data queue. In such a dynamic environment, static approaches no longer work for IIoT, for example, \cite{sattler2020clustered}, and the K-Center clustering algorithm in \cite{wang2020optimizing}.
\item \textbf{Natural geographical clustering property.} In modern industrial parks, IIoT devices in each factory are geographically adjacent, which makes them naturally clustered into groups and interconnected by highly reliable networks, for example, through regional 5G base stations (see Fig. \ref{fig:industrial-park}). However, this valuable property is often ignored, and the rich communication resources at the edge are not fully utilized\cite{zhang2021client,zeng2022heterogeneous,nishio2019client}, which limits the improvement of industrial FL.
\end{enumerate}}

\new{The above characteristics distinguish ``FL in IIoT" from ``FL in non-IIoT" (e.g., FL in Edge). Few literature has been proposed to tackle the non-i.i.d. data challenge of FL in IIoT, such as approaches based on centroid distance weighted averaging\cite{zhang2020blockchain}, reinforcement learning\cite{pang2020realizing}, and kmeans-based cohorts\cite{hiessl2021cohort}. However, none of them take into account the changing local data distribution or the natural geographic clustering property of devices in IIoT. More importantly, they do not address the fundamental problem causing FL model performance to degrade, namely the divergence in class distributions\cite{zhao2018federated}}.

\begin{figure}[t]
\centering
\includegraphics[width=0.49\textwidth]{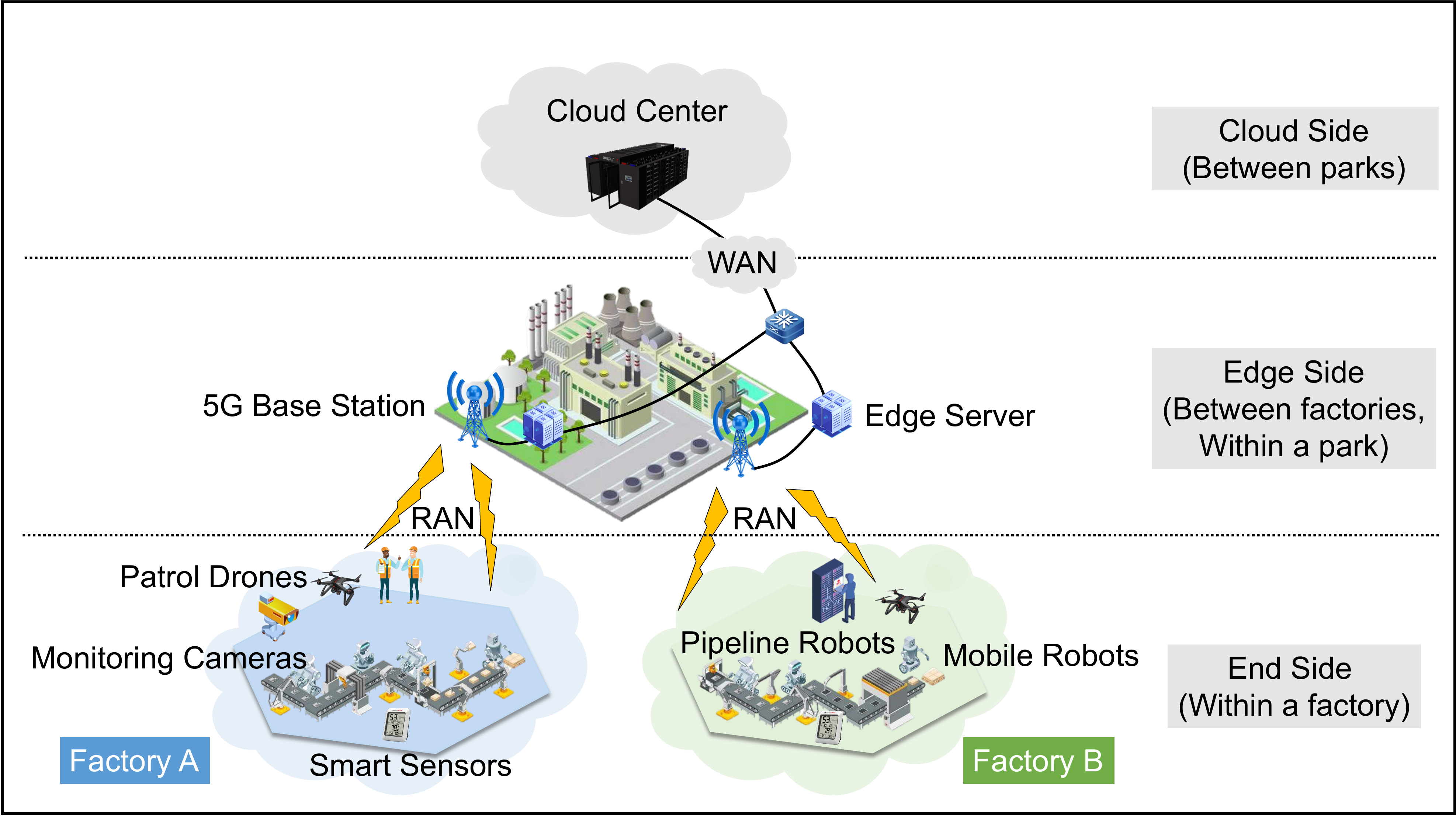}
\caption{A network architecture example in the modern industrial park. \new{End devices within each factory submit locally trained ML models to nearby 5G edge base stations. Edge servers synchronize these models and upload the synchronized models to the cloud server for global synchronization. The cloud center can be located on the cloud to synchronize ML models from multiple industrial parks, or it can be a micro data center located in an industrial park to synchronize ML models of edge servers in this park.}}
\label{fig:industrial-park}
\end{figure}

\new{To address the root cause of non-i.i.d. data, this paper aims to propose an effective approach to minimize the divergence in class distributions among heterogeneous devices. Taking advantage of the natural property of geographical clustering, we can select a subset of devices in each factory to construct ``FL super nodes" with consistent class distributions. These super nodes can be treated as homogeneous clients participating in FL training, without exposing the confidential industrial FL process in risky data manipulation.} However, designing such an approach is not trivial. Firstly, selecting a subset of devices in each group to minimize the class distribution divergence among groups is a 0-1 integer programming problem with vector weight constraints, which is proved to be NP-complete. \new{More challenging, this procedure needs to be invoked frequently to adapt to rapidly changing local data and the mobility of mobile IIoT devices (e.g., robots, drones), which places high demands on execution latency.} Secondly, even if class distributions among FL super nodes are forced to be homogeneous, devices' local data in each FL super node can still be skewed. If not handled properly, these challenges will still degrade FL model performance.

\new{To minimize the data heterogeneity among groups and realize efficient client selection, this paper proposes a novel gradient-based binary permutation optimizer \SAMPLERNAME~to solve the above NP-complete client selection problem. \SAMPLERNAME~runs a constraint-preserving gradient descent optimization procedure directly in the 0-1 integer space, and can build homogeneous FL super nodes in a very short time. Then, we propose \underline{Fed}erated \underline{G}roup \underline{S}ynchronization (\NAME), which is a hierarchical cloud-edge-end FL framework for 5G empowered modern industries, to improve industrial FL performance on non-i.i.d. data. \NAME~uses a compound-step synchronization protocol to train ML models, which can suppress data heterogeneity within and among FL super nodes. More specifically, \NAME~uses a single-step synchronization protocol (e.g., SSGD\cite{zinkevich2010parallelized}) within super nodes because of its robustness against data heterogeneity, and a multi-step synchronization protocol (e.g., FedAvg\cite{mcmahan2016communication}) among homogeneous super nodes to reduce communication overhead. Theoretical analysis shows that \NAME~has both the convergence upper bound and optimality gap better than FedAvg in the presence of non-i.i.d. data, and can be more time-efficient under a relaxed condition.} Finally, we evaluate \NAME~on the most widely adopted non-i.i.d. benchmark dataset FEMNIST\cite{caldas2018leaf}, and compare it with 10 advanced approaches, including FedAvg, FedMMD\cite{yao2018two}, FedFusion\cite{yao2019towards}, FedProx\cite{li2018federated}, IDA\cite{yeganeh2020inverse}, CGAU\cite{rieger2020client}, FedAvgM\cite{hsu2019measuring}, and FedAdagrad, FedAdam, FedYogi from \cite{reddi2020adaptive}. The main contributions of this paper are summarized as follows.

\begin{itemize}
\item \new{We propose a hierarchical cloud-edge-end FL framework \NAME~for 5G empowered modern industries, which uses a novel compound-step synchronization protocol to coordinate the training process within and among groups. The new protocol is robust against data heterogeneity and can effectively improve industrial FL performance.}
\item \new{We propose a novel \SAMPLERNAME~algorithm to select a subset of devices from each group to build homogeneous FL super nodes, which can find a desirable selection strategy in a very short time. \SAMPLERNAME~is a general optimizer for constrained 0-1 integer programming problems and can be used for other practical cases such as game matching.}
\item \new{We analyze the convergence rate and optimality gap of \NAME~and give a relaxed condition under which \NAME~is more time-efficient than FedAvg. Theoretical results show that \NAME~not only converges closer to the optimal, but also faster.}
\item Extensive experiments compared to 10 advanced approaches show that \NAME~improves FL accuracy by 3.5\% and reduces training rounds by 59\% on average. The results highlight the superior effectiveness and efficiency of \NAME~on non-i.i.d. data.
\end{itemize}
\section{Related Work}\label{section:related-work}
In this section, we categorize related works into four types according to the techniques they use.

\textbf{Data Sharing and Augmentation.}
This type of approach aims to minimize the class distribution divergence among devices by sharing or augmenting FL clients' local datasets. For the sharing-based approaches, Zhao \textit{et al.} propose to distribute a small portion of globally shared data (e.g., open available data) to clients' devices\cite{zhao2018federated}. Yao \textit{et al.} collect metadata shared by voluntary clients to perform controllable meta updating\cite{yao2019federated}. Yoshida \textit{et al.} reward FL clients for contributing local datasets and propose a hybrid learning mechanism wherein the server updates the model using the shared data and clients' local models\cite{yoshida2019hybridfl}. \new{These approaches achieve a considerable improvement in FL accuracy, but they are suspected of leaking private data due to the need to share clients' local datasets. Besides, open-available datasets do not always exist, especially in fields where data is highly confidential.}

For the augmentation-based approaches, Duan \textit{et al.} observe that the imbalance among different classes can also degrade FL accuracy\cite{duan2019astraea}. Hence, they augment classes with fewer samples by simply random offset, rotation, cropping, and scaling. Jeong \textit{et al.}\cite{jeong2018communication} propose to generate new samples using a globally trained conditional generative adversarial network (CGAN) to build unskewed local datasets. \new{Similarly, Wang \textit{et al.} generate synthetic data in the minority class based on linear interpolation to re-balance local datasets on edge devices\cite{wang2021non}.} \new{These approaches avoid the leakage of FL clients' private data. However, they still bring credibility crises. Speculative clients can use the synthetic data generated out of thin air to participate in FL training while hiding their original data. Also, they can pretend that the synthetic data is a large volume of high-quality data for more rewards. Therefore, these operations (i.e., data sharing, data augmentation) are high risk and should be prohibited.}

\textbf{Hyperparameter Tuning.}
Hyperparameters play an important role in FL training convergence. Some works have been explored in hyperparameter tunings, such as tunning the number of local iterations and the learning rate. Wang \textit{et al.} point out that the optimal performance can be achieved when the number of local iterations equals one\cite{wang2019adaptive}. However, the constrained resources (e.g., bandwidth, time, power) prevent us from doing this. In practice, large local iterations are more commonly used. For example, Yu \textit{et al.} carefully set the number of local iterations and obtain a considerable convergence rate\cite{yu2019parallel}. In addition, Li \textit{et al.} point out that decaying the learning rate is necessary for FL convergence with large local iterations\cite{li2019convergence}. For a strongly convex and smooth objective function, FedAvg can converge to the optimal after applying learning rate decay, with a convergence rate of $\mathcal{O}(1/\mathcal{T})$, where $\mathcal{T}$ is the total number of local updates on a single device. \new{These works give rigorous proofs for convergence analysis, which guides follow-up optimization on FL. However, carefully tuning these hyperparameters (e.g., number of local iterations, learning rate, and decay rate) requires multiple attempts and incurs high time costs.}

\textbf{Client-Side Adaption.}
This type of approach emphasizes that FL clients should adaptively retain global knowledge while improving local knowledge. Some examples are given to integrate them. Yao \textit{et al.} point out that the global model contains more global knowledge and should be kept as a reference, rather than simply thrown away. Based on this idea, they adopt a two-stream model to transfer the global knowledge to the local model\cite{yao2018two}. By minimizing the maximum mean discrepancy (MMD) loss, the two-stream model can extract more generalized features and learn better local representations. Then, in \cite{yao2019towards}, they use the $1\times1$ convolution, vector weighted average, and scalar weighted average operators to fuse the global and local features. Li \textit{et al.} point out that too many local updates will cause the FL training to diverge, especially under the non-i.i.d. data setting\cite{li2018federated}. Hence, they add a proximal penalty term to local objective loss functions to constrain the local model to be closer to the global model and avoid excessive divergence. Rieger \textit{et al.} point out that clients express representations in different patterns and their shared knowledge may be obfuscated after synchronization\cite{rieger2020client}. Hence, they adopt conditional gated activation units to enable clients to condition their units. In this way, clients can identify whether the global feature is expressed and how to modulate the global pattern. \new{These approaches impose more storage footprint and computation on resource-constrained client-side devices, requiring higher resource allocation and also higher energy consumption.}

\textbf{Server-Side Adaption.}
This type of approach explores how local models can be adaptively aggregated and how the global model can be adaptively optimized on the server-side. Yeganeh \textit{et al.} aggregate clients' local models according to their weights, by capitalizing an adaptive weighting approach based on the inverse distance between the local model parameters and the averaged model parameters\cite{yeganeh2020inverse}. By using this approach, out-of-distribution models will be weighed down and the global model can have a lower variance. The authors also explored the combination with other metrics, such as the training accuracy and the data size. \new{However, these variants did not perform well in our experiments, probably because some honest but ``out-of-distribution" devices were over-suppressed.} On the other hand, inspired by the ability of momentum accumulation to dampen oscillations\cite{nesterov2013gradient}, Hsu \textit{et al.} adopt the momentum optimizer on the server-side and observe a significant improvement in FL accuracy\cite{hsu2019measuring}. Then, Reddi \textit{et al.} introduce three advanced adaptive optimizers (i.e., Adagrad\cite{ward2019adagrad}, Adam\cite{kingma2014adam}, and Yogi\cite{zaheer2018adaptive}) to update the server-side global model\cite{reddi2020adaptive}. \new{These adaptive federated optimizers enable the use of adaptive learning rates for different gradients and achieve great success, but unfortunately, they also require careful tunning of initial learning rates, and we observed drastic accuracy oscillation in the experiments.}
\section{System Model}

\begin{table}
  \caption{Summary of Main Symbols.}
  \label{table:symbols}
  \centering
  \renewcommand{\arraystretch}{1}
  \begin{tabular}{p{30pt}<{\centering}|p{194pt}<{}}
    \hline
    \textbf{Symbol} & \textbf{Explanation} \\
    \hline\hline
    $\mathcal{L}$ & Loss function (e.g., cross-entropy). \\
    \hline
    $R$ & Maximum training rounds. \\
    \hline
    $T$ & Number of iterations in each round. \\
    \hline
    $K$, $K^m$ & Number of devices (in factory $m$). \\
    \hline
    $L$ & Number of devices to be selected per factory. \\
    \hline
    $L_{\mathrm{rnd}}$ & Number of devices to be randomly pre-sampled per factory. \\
    \hline
    $L_{\mathrm{sel}}$ & Number of devices to be selected by \SAMPLERNAME~per factory. \\
    \hline
    $M$ & Number of factories (also the number of groups). \\
    \hline
    $F$ & Number of classification classes. \\
    \hline
    $\eta$ & Local learning rate. \\
    \hline
    $\omega_t$ & Parameters of the global ML model at $t$-th iteration. \\
    \hline
    $\omega_t^{m}$ & Parameters of the ML model on BS $m$ at $t$-th iteration. \\
    \hline
    $\omega_t^{m,k}$ & Parameters of the local ML model on device $k$ in factory $m$ at $t$-th iteration. \\
    \hline
    $\mathcal{D}^{m,k}$ & Local dataset of device $k$ in factory $m$. \\
    \hline
    $\mathcal{D}^{m,k}_{t}$ & A mini-batch data of device $k$ in factory $m$ at $t$-th iteration. \\
    \hline
    $N^{m,k}$ & Size of local dataset of device $k$ in factory $m$. \\
    \hline
    $n$, $n^{m,k}$ & Batch data size (of device $k$ in factory $m$). \\
    \hline
    $n^m$ & Total size of data batches in factory $m$. \\
    \hline
    $\mathbf{a}^{m,k}_t$ & Data size vector of $F$ label classes of mini-batch data $\mathcal{D}^{m,k}_{t}$. \\
    \hline
    $\mathbf{A}^m_t$ & Data size matrix of $\mathbf{a}^{m,k}_t$ of $K^m-L_{\mathrm{rnd}}$ devices. \\
    \hline
    $\mathbf{b}^m_t$ & Total data size vector of the pre-sampled $L_{\mathrm{rnd}}$ devices. \\
    \hline
    $\mathcal{C}^m$ & Set of all devices in factory $m$. \\
    \hline
    $\mathcal{C}^m_t$ & Set of $L$ selected devices in factory $m$ at $t$-th iteration. \\
    \hline
    $\mathcal{P}^{m,k}$ & Local data distribution of device $k$ in factory $m$. \\
    \hline
    $\mathcal{P}^{m,k}_t$ & Local data distribution of mini-batch data $\mathcal{D}^{m,k}_{t}$. \\
    \hline
    $\mathcal{P}^{m}_t$ & Mean data distribution of $\mathcal{P}^{m,k}_t$ over selected devices $\mathcal{C}^m_t$. \\
    \hline
    $\mathcal{P}^{\mathrm{real}}$ & Real-world global data distribution. \\
    \hline
  \end{tabular}
\end{table}

IoT devices in modern industrial parks can be divided into two types: Fixed devices (e.g., monitoring cameras, temperature and humidity sensors) and mobile devices (e.g., patrol drones and logistics robots). In the industrial park, due to the advantages of improved performance, reduced communication cost, and decentralized scalability, FL plays an important role in many industrial applications. For example, an anomaly detection application based on on-device federated monitors could be applied for IIoT scenarios, where sensing and monitoring devices may locate in harsh environments with high voltage and high radiation, and they may move around the factory, making them impractical to access wired networks\cite{liu2020deep}. 5G mobile networks have been considered to be enhanced to support key performance features of industrial applications such as high throughput, low latency, and high scalability\cite{akpakwu2017survey}. These features enable industrial FL applications to transmit model data of a large number of IIoT devices at a high cycle frequency, with a high data rate and low latency. Therefore, we consider a hierarchical cloud-edge-end network architecture empowered by 5G cellular wireless networks for training industrial FL applications, as shown in Fig. \ref{fig:industrial-park}. 


In this case, a modern industrial park has $M$ factories, each factory $m$ has $K^m$ smart devices. We consider the devices in the same factory as a group. These devices are equipped with embedded computing chips to perform lightweight processing, for example, training ML models on local data streams. The devices connect to the nearby base station (BS) (also identified as $m$) through 5G cellular wireless networks. Then, BSs communicate FL model data with the cloud center through the Internet. \new{Specifically, the device $k$ of factory $m$ collects streaming sensory data in real-time and changes the local dataset $\mathcal{D}^{m,k}$, whose class distribution is defined as $\mathcal{P}^{m,k}$.} The class distributions of different devices can be highly heterogeneous due to diverse local usage patterns. Thus, we have $\mathcal{P}^{m_1,k_1}\ne\mathcal{P}^{m_2,k_2}\ne\mathcal{P}^{\mathrm{real}} (\forall m_1\ne m_2, k_1\ne k_2)$, where $\mathcal{P}^{\mathrm{real}}$ is the real-world global class distribution. The goal of industrial FL is to find the optimal model parameters $\omega^*$ that can minimize the global loss function,
\begin{equation}
\omega ^*\triangleq \underset{\omega}{\mathrm{arg}\min}\sum_{m=1}^{M}{\sum_{k=1}^{K^m}{\mathcal{L}\left( \omega ,\mathcal{D}^{m,k} \right)}}.
\end{equation}

The traditional workflow is briefly described below. In each round, a small subset of devices is randomly selected to participate in FL training. These devices utilize local datasets for several epochs to train their local ML models and upload these local models to connected BSs. BSs aggregate these local models and upload the aggregated models to the top server in the cloud. The top server globally aggregates BSs' models, updates the global ML model, and synchronizes the updated model back to all BSs and end devices. These steps are repeated until the global model parameters $\omega ^*$ converge. However, this approach causes performance degradation to industrial FL due to data heterogeneity among devices, and it ignores the streaming nature of industrial data. Hence, in Section \ref{section:framework}, a data heterogeneity-robust federated group synchronization approach is presented to address this issue.

We summarize the main symbols used in this paper in Table \ref{table:symbols}. The framework and workflow of the proposed \NAME~are illustrated in Fig. \ref{fig:framework}.
\begin{figure}[t]
\centering
\includegraphics[width=0.5\textwidth]{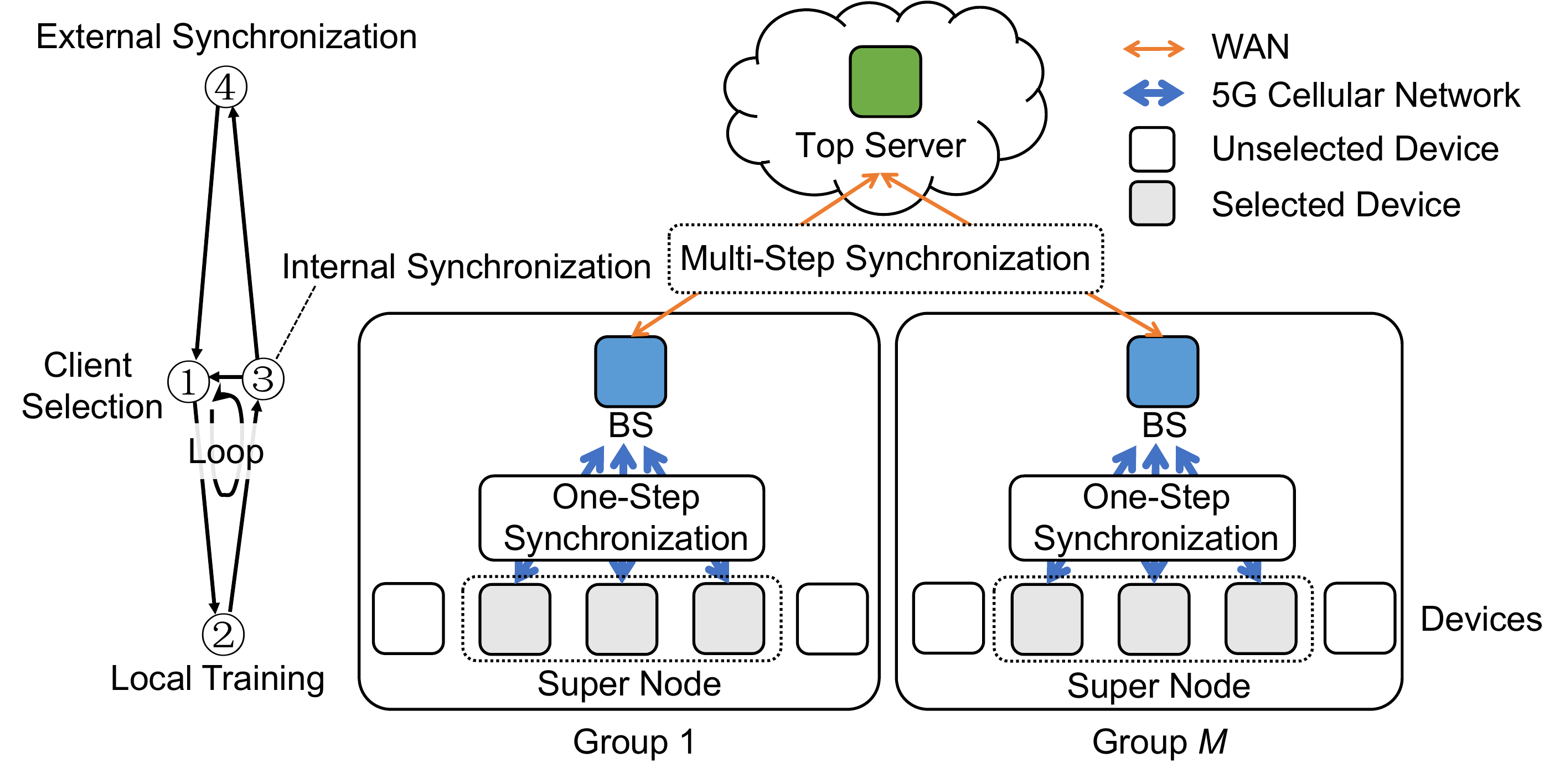}
\caption{Overview of \NAME~framework and workflow. \textcircled{1} Each BS selects $L$ devices to form a super node. \textcircled{2} Each selected device trains its local model for one mini-batch SGD step. \textcircled{3} Each BS synchronizes local models in its group. \textcircled{4} Top server synchronizes models of BSs. \textcircled{1}\textcircled{2}\textcircled{3} loop $T$ times every \textcircled{4}.}
\label{fig:framework}
\end{figure}

\section{\NAME: Framework and Workflow}\label{section:framework}
The core idea is to strategically select a small subset of devices in each group to form FL super nodes with homogeneous data distributions. Then, these super nodes can be regarded as homogeneous clients to participate in FL. \new{To resolve the heterogeneity in local datasets of devices inside each super node, the one-step synchronization protocol (e.g., SSGD) can be useful because it was proved to be equivalent to centralized SGD, which gives it robustness against data heterogeneity. Meanwhile, the multi-step synchronization protocol (e.g., FedAvg) can be used to keep \NAME~communication efficient, since the class distributions of super nodes are aligned. In this way, the problem of data heterogeneity is decomposed from the entire population of devices to a small number of devices in multiple groups, making it efficiently and effectively solved by the compound-step synchronization protocol. In this way, the performance degradation of industrial FL is addressed.}

The detailed design is given in Alg. \ref{alg:fedgs}. In the initialization stage, the top server first initializes the global ML model parameters $\omega_0$ and synchronizes $\omega^m_0\gets\omega_0$ to BSs. Then, it collects local class distributions $\mathcal{P}^{m,k} (\forall m,k)$ from all devices to estimate the real-world global class distribution $\mathcal{P}^{\mathrm{real}}$,
\begin{equation}\label{eq:p_real}
\mathcal{P}^{\mathrm{real}}=\mathrm{norm}\left( \sum_{m\in M}{\sum_{k\in K^m}{N^{m,k}\mathcal{P}^{m,k}}} \right), 
\end{equation}
where $N^{m,k}$ is the local data size of device $k$ in group $m$, and $\mathrm{norm}(\cdot)$ is a probability normalization function.

\textbf{Client Selection.}
In each iteration $t$, each BS $m$ selects $L$ devices from its group $\mathcal{C}^m$ to obtain a homogeneous super node $\mathcal{C}^m_t$ via the Select-Clients-Via-\SAMPLERNAME~interface. The detailed algorithm \SAMPLERNAME~is presented in Section \ref{section:client-selection}.

\textbf{Local Training.} In each group $m$, each selected device $k$ fetches a mini-batch of data $\mathcal{D}^{m,k}_t$ from local dataset $\mathcal{D}^{m,k}$ with batch size $n^{m,k}$. \new{These mini-batch streaming data are one-shot and will not be used again.} Then, the device $k$ downloads the model $\omega^{m,k}_{t-1}\gets\omega^{m}_{t-1}$ from the connected BS $m$ and trains $\omega^{m,k}_{t-1}$ for one mini-batch gradient descent step with learning rate $\eta$, 
\begin{equation}
\label{eq:sgd}
\omega _{t}^{m,k}\gets \omega _{t-1}^{m,k}-\frac{\eta}{n^{m,k}}\nabla _{\omega}\mathcal{L}\left( \omega _{t-1}^{m,k},\mathcal{D}_{t}^{m,k} \right).
\end{equation}

\textbf{Internal Synchronization.} The locally trained model $\omega^{m,k}_{t}$ will be uploaded to BS $m$ for internal synchronization,
\begin{equation}
\label{eq:internal-aggr}
\omega _{t}^{m}\gets \sum_{k\in \mathcal{C}_{t}^{m}}{\frac{n^{m,k}}{n^m}\omega _{t}^{m,k}}, 
\end{equation}
where $n^m=\sum_{k\in \mathcal{C}_{t}^{m}}{n^{m,k}}$ is the total data size of all used mini-batches in $\mathcal{C}_t^m$. Then, the BS $m$ updates its model with $\omega _{t}^{m}$ and synchronizes $\omega _{t}^{m}$ in its group. 

We call the above client selection, local training, and internal synchronization as a \textit{one-step synchronization iteration} because the local update on each device is only performed once before each synchronization. The one-step synchronization will loop $T$ iterations before each round of external synchronization. 

\textbf{External Synchronization.} For every time that the one-step synchronization is performed $T$ iterations, BSs can upload their model $\omega _{t}^{m}$ to the top server for global aggregation,
\begin{equation}
\label{eq:external-aggr}
\omega _t\gets \frac{1}{M}\sum_{m\in M}{\omega _{t}^{m}}.
\end{equation}
The globally aggregated model $\omega _t$ will be used to update the global model on the top server and synchronized to BSs.

Since the external synchronization is performed every $T$ one-step synchronization iterations, we call it a \textit{multi-step synchronization round}. The above steps will be repeated $R$ rounds (i.e., $TR$ iterations) to obtain the converged model parameters $\omega^*\gets\omega_{TR}$.

The above workflow can be seen as an equivalent version of FedAvg, which performs local updates on FL super nodes for $T$ iterations with larger batch sizes, but homogeneous local datasets among super nodes. By capitalizing on an effective client selection strategy to make these super nodes homogeneous, the FL training process can be robust against data heterogeneity, and FL model performance can be improved. In the following section, we give our solution \SAMPLERNAME.

\begin{algorithm}[t]
\caption{Federated-Group-Synchronization (Main)}\label{alg:fedgs} 
\begin{algorithmic}[1] 
\Require Number of iterations in each round $T$; Maximum training rounds $R$; Number of groups $M$; Number of selected devices per group $L$.
\Ensure Well-trained global FL model $\omega_{TR}$.
\State Initialize $\omega_0$ and $\omega_0^m\gets\omega_0$ and estimate $\mathcal{P}^{\mathrm{real}}$ by Eq. \eqref{eq:p_real};
\For{each internal synchronization $t=1,\cdots,TR$}
\For{each BS $m$ in $1,\cdots,M$ in parallel}
\State \textbf{Client Selection:} Select $L$ devices from group $\mathcal{C}^m$ to form a homogeneous super node $\mathcal{C}^m_t$: $\mathcal{C}^m_t\gets$ \colorbox{green}{Select-Clients-Via-\SAMPLERNAME}($L, \mathcal{C}^m,\mathcal{P}^{\mathrm{real}}$);
\For{each device $k$ \textbf{in} $\mathcal{C}^m_t$ in parallel}
\State \textbf{Local Training:} Fetch a mini-batch of data $\mathcal{D}_t^{m,k}$ and update the local model $\omega^{m,k}_{t}$ by Eq. \eqref{eq:sgd};
\EndFor
\State \textbf{Internal Synchronization:} BS $m$ aggregates local models $\{\omega^{m,k}_t|\forall k\in \mathcal{C}^m_t\}$ by Eq. \eqref{eq:internal-aggr} and update $\omega^{m}_t$;
\If{$t\%T==0$}
\State \textbf{External Synchronization:} The top server globally aggregates $\{\omega^m_t|\forall m\}$ by Eq. \eqref{eq:external-aggr}, and synchronizes the updated global model $\omega_t$ to BSs: $\omega^m_{t}\gets\omega_{t}$;
\EndIf
\EndFor
\EndFor
\State \Return $\omega_{TR}$;
\end{algorithmic} 
\end{algorithm}
\section{Client Selection Via \SAMPLERNAME}\label{section:client-selection}
In this section, we formulate the client selection problem as a 0-1 integer programming problem with vector weight constraints, and present our novel Gradient-based Binary Permutation approach, namely \SAMPLERNAME, to solve this problem in an acceptable short time, with a desirable solution.

\subsection{Problem Modeling}
Given a factory (i.e., group) $m$ with $K^m$ industrial devices $\mathcal{C}^m$ ($\left|\mathcal{C}^m\right|=K^m$). The next batch data of device $k$ follows the distribution $\mathcal{P}^{m,k}_t$ ($\ne \mathcal{P}^{\mathrm{real}}$) and the data size vector of $F$ label classes is $\mathbf{a}^{m,k}_t=n^{m,k}\mathcal{P}^{m,k}_t\in \mathbb{Z}^{F\times 1}$. \new{Our goal is to select $L$ devices from group $m$ at iteration $t$ to form a FL super node $\mathcal{C}^m_t$} whose overall class distribution $\mathcal{P}^m_t$ satisfies,
\begin{equation}\label{eq:goal}
\underset{\mathcal{P}^m_t\triangleq \mathbb{E}\left[ \left\{ \mathcal{P}^{m,k}_t|\forall k\in \mathcal{C}_{t}^{m} \right\} \right]}{\min}\,\,\left\| \mathcal{P}^m_t-\mathcal{P}^{\mathrm{real}} \right\| _{L_2}.
\end{equation}

\new{Note that if $\mathcal{P}^{m,k}_t$ is fixed, Eq. \eqref{eq:goal} will always find a fixed device set $\mathcal{C}^m_t$ and other devices $\mathcal{C}^m\backslash \mathcal{C}^m_t$ will have no chance to be selected. To keep the randomness of client selection so that each device has the same probability to be selected, we use a trick to randomly pre-sample $L_{\mathrm{rnd}}$ devices before strategically selecting the remaining $L_{\mathrm{sel}}=L-L_{\mathrm{rnd}}$ devices.} Formally speaking, in group $m$, we first sample $L_{\mathrm{rnd}}$ devices at random to obtain $\mathcal{C}^m_{\mathrm{rnd}}$, whose next data batches have a total data size of vector $\mathbf{b}^m_t\in \mathbb{Z}^{F\times 1}$. Then, $L_{\mathrm{sel}}$ devices are further selected from the remaining $K^m-L_{\mathrm{rnd}}$ devices $\mathcal{C}^m\backslash\mathcal{C}^m_{\mathrm{rnd}}$, whose data size matrix is $\mathbf{A}^m_t=\left[ \mathbf{a}^{m,1}_t,\mathbf{a}^{m,2}_t,\cdots ,\mathbf{a}^{m,K^m-L_{\mathrm{rnd}}}_t \right] \in \mathbb{Z}^{F\times \left( K^m-L_{\mathrm{rnd}} \right)}$, with the goal to minimize Eq. \eqref{eq:goal}.

We use the following mathematical model to describe the above problem. Let $\mathbf{e}_{K^m-L_{\mathrm{rnd}}}^{T}\in 1^{1\times \left( K^m-L_{\mathrm{rnd}} \right)}$ and $\mathbf{e}_{F}^{T}\in 1^{1\times F}$, the objective is to find a solution $\mathbf{x}^m_t\in \mathbb{Z}^{\left( K^m-L_{\mathrm{rnd}} \right) \times 1}$, where $\mathbf{x}_t^{m}(i)\in\{0,1\}$ and $\mathbf{e}_{K^m-L_{\mathrm{rnd}}}^{T}\cdot \mathbf{x}^m_t=L_{\mathrm{sel}}$, that
\begin{align}
\label{eq:complex-goal}
\underset{\mathbf{x}^m_t}{\min}\quad &\left\| \frac{\mathbf{A}^m_t\mathbf{x}^m_t+\mathbf{b}^m_t}{\mathbf{e}_{F}^{T}\left( \mathbf{A}^m_t\mathbf{x}^m_t+\mathbf{b}^m_t \right)}-\mathcal{P}^{\mathrm{real}} \right\| _{L_2},
\\
\mathrm{s}.\mathrm{t}.\quad &\mathbf{x}_t^{m}(i)\in \left\{ 0,1 \right\}, 
\\
&\mathbf{e}_{K^m-L_{\mathrm{rnd}}}^{T}\cdot \mathbf{x}^m_t=L_{\mathrm{sel}}.
\end{align}
Let batch sizes of all data batches be the same $n=n^{m,k}\,\,\left( \forall m,k \right) $, then we have $\mathbf{e}_{F}^{T}\left( \mathbf{A}^m_t\mathbf{x}^m_t+\mathbf{b}^m_t \right) =nL$ and a simplified model,
\begin{align}
\label{eq:simple-goal}
\underset{\mathbf{x}^m_t}{\min}\quad &\left\| \mathbf{A}^m_t\mathbf{x}^m_t-\mathbf{y}^m_t \right\| _{L_2},
\\
\label{eq:y}
\mathrm{s}.\mathrm{t}.\quad &\mathbf{y}^m_t=nL\mathcal{P}^{\mathrm{real}}-\mathbf{b}^m_t,
\\
\label{eq:0-1-constraint}
&\mathbf{x}_t^{m}(i)\in \left\{ 0,1 \right\},
\\
\label{eq:weight-constraint}
&\mathbf{e}_{K^m-L_{\mathrm{rnd}}}^{T}\cdot \mathbf{x}^m_t=L_{\mathrm{sel}}.
\end{align}

\new{Note that in the mathematical model above, we considered data size and data quality in the objective (i.e., minimizing the distribution divergence) for client selection, but assumed that IIoT devices have similar hardware capabilities. However, if system heterogeneity should be considered, ESync\cite{li2020esync} is compatible and can be useful.} Then, we give a simple proof to show that the above problem is NP-complete.

\begin{lemma}[Problem A]
\label{lemma:np-complete}
Given an integer matrix $\mathbf{A}$ and an integer vector $\mathbf{y}$, the goal is to find whether there is a 0-1 vector $\mathbf{x}$ such that  $\mathbf{Ax}=\mathbf{y}$. This 0-1 integer programming problem is NP-complete\cite{rice1981}.
\end{lemma}

\begin{proposition}[Problem B]
\label{prop:np-complete}
Let Lemma \ref{lemma:np-complete} hold and constrain the number of 1 in $\mathbf{x}$ to be $L_{\mathrm{sel}}$ (Eq. \eqref{eq:weight-constraint}). This variant problem is at least NP-complete.
\end{proposition}

\begin{proof}
To solve problem A, we can solve problem B for $K$ times with $L_{\mathrm{sel}}=1,2,\cdots,K$, where $K$ is the size of vector $\mathbf{x}$. This input transformation has a linear complexity $\mathcal{O}\left( K \right)$. Problem B outputs YES if Eq. \eqref{eq:simple-goal} could reach 0, otherwise, it outputs NO. This output transformation has a constant complexity $\mathcal{O}(1)$. Hence, problem A can reduce to problem B in a polynomial complexity, which makes problem B also NP-complete.
\end{proof}

We can see from Proposition \ref{prop:np-complete} that \new{it is almost impossible to find the optimal solution in a polynomial complexity. To make \NAME~time efficient, a sub-optimal but fast solution is preferred,} as described in the following subsection.

\subsection{Gradient-based Binary Permutation Client Selection}
To make the NP-complete client selection problem solvable, in this paper, we propose a novel gradient-based approximate approach, namely \SAMPLERNAME. The core idea is to permute (0,1) pairs of binary selection variables in $\mathbf{x}$ with the steepest opposite gradients. In other words, the variable $\mathbf{x}(i)$ with the selection value of 0 and the smallest gradient will be permuted with the variable $\mathbf{x}(j)$ with the selection value of 1 and the largest gradient. In this way, the number of variables whose selection value equals one (i.e., the vector weight constraint) is maintained, and constraints \eqref{eq:0-1-constraint} and \eqref{eq:weight-constraint} can be satisfied. This binary permutation operation will be performed iteratively to minimize the objective Eq. \eqref{eq:simple-goal}.

\begin{algorithm}[t] 
\caption{\colorbox{green}{Select-Clients-Via-\SAMPLERNAME}} 
\label{alg:gbp-cs} 
\begin{algorithmic}[1] 
\Require Number of selected devices per group $L$; Device set $\mathcal{C}^m$ of group $m$; Global class distribution $\mathcal{P}^{\mathrm{real}}$.
\Ensure $L$ selected devices $\mathcal{C}^m_t$.
\State Pre-sample $L_{\mathrm{rnd}}$ devices $\mathcal{C}^m_{\mathrm{rnd}}$ at random, and construct $\textbf{b}^m_t$ from $\mathcal{C}^m_{\mathrm{rnd}}$ and $\textbf{A}^m_t$ from $\mathcal{C}^m\backslash\mathcal{C}^m_{\mathrm{rnd}}$;
\State Initialize $s\gets 1$, $\mathbf{y}^m_t$ by Eq. \eqref{eq:y}, and $\mathbf{x}_1$ by Eq. \eqref{eq:mp-init-solution};
\State Calculate the distance $d_s=\left\| \mathbf{A}^m_t\mathbf{x}_s-\mathbf{y}^m_t \right\| _{L_2}$;
\Repeat
\State Calculate the gradient $\boldsymbol{g}_s=\nabla _{\mathbf{x}}d_s$;
\State Select an index pair $\left( i_{0\rightarrow 1},i_{1\rightarrow 0} \right)$ by Eq. \eqref{eq:i-0to1}-\eqref{eq:i-1to0};
\State Make a copy $\mathbf{x}_{s+1}\gets\mathbf{x}_{s}$ and permute $\mathbf{x}_{s+1}\left( i_{0\rightarrow 1} \right)$ and $\mathbf{x}_{s+1}\left( i_{1\rightarrow 0} \right)$ by Eq. \eqref{eq:permute};
\State Update the distance $d_{s+1}=\left\| \mathbf{A}^m_t\mathbf{x}_{s+1}-\mathbf{y}^m_t \right\| _{L_2}$;
\State $s\gets s+1$;
\Until{$d_s > d_{s-1}$;}
\State Construct the set of $L$ selected devices $\mathcal{C}^m_t=\mathcal{C}^m_{\mathrm{rnd}}\cup\mathcal{C}^m_{\mathrm{sel}}$, where $\mathcal{C}^m_{\mathrm{sel}}$ is defined as Eq. \eqref{eq:C_sel};
\State \Return $\mathcal{C}^m_t$;
\end{algorithmic} 
\end{algorithm}

The pseudo code of \SAMPLERNAME~is given in Alg. \ref{alg:gbp-cs}. Given the data size matrix $\mathbf{A}^m_t$ and $\mathbf{y}^m_t=nL\mathcal{P}^{\mathrm{real}}-\mathbf{b}^m_t$, \SAMPLERNAME~first initializes the solution variable $\mathbf{x}_1$ as follows,
\begin{equation}\label{eq:mp-init-solution}
\mathbf{x}_1\triangleq\left\{\mathcal{T}_{L_{\mathrm{sel}}}(\tilde{\mathbf{x}}_1) | \tilde{\mathbf{x}}_1={(\mathbf{A}^m_t)}^{-1}\mathbf{y}^m_t\right\},
\end{equation}
where ${(\mathbf{A}^m_t)}^{-1}\mathbf{y}^m_t$ is the Moore-Penrose Inverse solution, $\mathcal{T}_{L_{\mathrm{sel}}}(\tilde{\mathbf{x}}_1)$ means to set the largest $L_\mathrm{sel}$ values of $\tilde{\mathbf{x}}_1$ to 1, and the others to 0. Then, \SAMPLERNAME~calculates the objective distance $d_s=\left\| \mathbf{A}^m_t\mathbf{x}_s-\mathbf{y}^m_t \right\| _{L_2}$ and the gradient $\boldsymbol{g}_s=\nabla _{\mathbf{x}}d_s$. 

The gradient $\boldsymbol{g}_s(i)$ indicates the opposite direction in which $\mathbf{x}_s(i)$ should be updated. The greater the absolute value of $\boldsymbol{g}_s\left( i \right)$, the smaller the $d_s$ can be obtained by updating $\mathbf{x}_s\left( i \right)$, so $\boldsymbol{g}_s\left( i \right)$ appears as a key gradient\cite{zhou2021dgt}. Based on this idea, \SAMPLERNAME~selects a pair of selection variables ($\mathbf{x}_s\left( i_{0\rightarrow 1} \right) , \mathbf{x}_s\left( i_{1\rightarrow 0} \right) $) with opposite key gradients for permutation. More specifically, $i_{0\rightarrow 1}$ is the identity of the device with the selection value of 0 and the smallest gradient,
\begin{equation}\label{eq:i-0to1}
i_{0\rightarrow 1}\triangleq\underset{i}{\mathrm{arg}\min}\left\{ \boldsymbol{g}_s\left( i \right) |\mathbf{x}_s\left( i \right) =0, \forall i\in[1,K^m-L_{\mathrm{rnd}}] \right\}.
\end{equation}
Similarly, $i_{1\rightarrow 0}$ is the identity of the device with the selection value of 1 and the largest gradient,
\begin{equation}\label{eq:i-1to0}
i_{1\rightarrow 0}\triangleq\underset{i}{\mathrm{arg}\max}\left\{ \boldsymbol{g}_s\left( i \right) |\mathbf{x}_s\left( i \right) =1, \forall i\in[1,K^m-L_{\mathrm{rnd}}] \right\}.
\end{equation}
Then, \SAMPLERNAME~permutes the values of $\mathbf{x}_s\left( i_{0\rightarrow 1} \right)$ and $\mathbf{x}_s\left( i_{1\rightarrow 0} \right)$ to obtain a new solution $\mathbf{x}_{s+1}$, 
\begin{equation}\label{eq:permute}
\mathbf{x}_{s+1}\left( i_{0\rightarrow 1} \right) =1,~
\mathbf{x}_{s+1}\left( i_{1\rightarrow 0} \right) =0.
\end{equation}
Eqs. \eqref{eq:i-0to1}-\eqref{eq:permute} will be repeated until the objective distance $d_s$ no longer decreases. Finally, we can construct
\begin{equation}\label{eq:C_sel}
\mathcal{C}^m_{\mathrm{sel}}\triangleq\{\mathrm{device}~i|\mathbf{x}^*(i)=1, \mathbf{x}^*=\mathbf{x}_{s-1}, \forall i\in[1,K^m-L_{\mathrm{rnd}}] \},
\end{equation}
and obtain the set of $L$ selected devices $\mathcal{C}^m_t=\mathcal{C}^m_{\mathrm{rnd}}\cup\mathcal{C}^m_{\mathrm{sel}}$.

\SAMPLERNAME~has a complexity of $\mathcal{O}(F^3+\alpha F^2+\alpha^2 F\tau+L)$, where $\alpha=K^m-L_{\mathrm{rnd}}$ and $\tau$ is the number of \SAMPLERNAME~iterations.
In our experiment, \SAMPLERNAME~can obtain a very desirable solution close to the optimal and has a high execution efficiency comparable to the random sampling approach.
\section{Performance Analysis}
In this section, we analyze the optimality gap and convergence rate of \NAME, and qualitatively compare them with those of FedAvg in the presence of non-i.i.d. data. Then, we give the condition under which \NAME~is time efficient.

\subsection{Convergence Analysis}
\begin{assumption}
The local function $\mathcal{L}$ are $\mu$-strongly convex, $\beta$-smooth, and $\rho$-Lipschitz. 
\end{assumption}

\begin{proposition}
For the gradient $\boldsymbol{g}_{t}^{m}$ on FL node $m$ in the federated setting and the gradient $\boldsymbol{g}_{t}^{c}$ in the centralized setting, we have an upper bound $\delta^m$ of the gradient divergence $\left\| \left( \mathcal{P}^m_t \right) ^T\cdot \boldsymbol{g}_{t}^{m}-\left( \mathcal{P}^{\mathrm{real}} \right) ^T\cdot \boldsymbol{g}_{t}^{c} \right\| $ proportional to the distribution divergence $\left\| \mathcal{P}^m_t-\mathcal{P}^{\mathrm{real}} \right\| $.
\end{proposition}
\begin{proof}
\begin{align}
\nonumber
&\left\| \left( \mathcal{P}^m_t \right) ^T\cdot \boldsymbol{g}_{t}^{m}-\left( \mathcal{P}^{\mathrm{real}} \right) ^T\cdot \boldsymbol{g}_{t}^{c} \right\| 
\\
\nonumber
=&\left\| \begin{array}{c}
	\left( \mathcal{P}^m_t \right) ^T\cdot \boldsymbol{g}_{t}^{m}-\left( \mathcal{P}^m_t \right) ^T\cdot \boldsymbol{g}_{t}^{c}\\
	+\left( \mathcal{P}^m_t \right) ^T\cdot \boldsymbol{g}_{t}^{c}-\left( \mathcal{P}^{\mathrm{real}} \right) ^T\cdot \boldsymbol{g}_{t}^{c}\\
\end{array} \right\| 
\\
\nonumber
\le& \left\| \left( \mathcal{P}^m_t \right) ^T\cdot \left( \boldsymbol{g}_{t}^{m}-\boldsymbol{g}_{t}^{c} \right) \right\| +\left\| \left( \mathcal{P}^m_t-\mathcal{P}^{\mathrm{real}} \right) ^T\cdot \boldsymbol{g}_{t}^{c} \right\| =\delta ^m.
\end{align}
\end{proof}

It is easy to know that $\delta^m$ captures the impact of divergence in class distributions $\left\| \mathcal{P}^m_t-\mathcal{P}^{\mathrm{real}} \right\| $. Generally speaking, the smaller the distribution divergence, the smaller the upper bound of the gradient divergence $\delta^m$. 

\begin{proposition}\label{prop:conv-bound}
Let $\delta \triangleq \mathbb{E}\left[ \delta ^m \right]$, the convergence upper bound of \NAME~is $\mathcal{O}( \frac{1}{R(T-\delta h(T))} ) $
, and its optimality gap is bounded by $\mathcal{O}( \frac{1}{TR}+\delta h(T)+o(\sqrt{\frac{\delta h(T)}{T}}) ) $.
\end{proposition}

\begin{proof}
As mentioned above, the convergence performance of \NAME~is theoretically equivalent to that of FedAvg, in which $M$ FL super nodes run mini-batch SGD with batch size $nL$ for $T$ local iterations in each round. Then, the convergence upper bound of \NAME~after $TR$ iterations can be inferred from Lemma 2 in \cite{wang2019adaptive},
$$\mathcal{L}\left( \omega _{TR} \right) -\mathcal{L}\left( \omega ^* \right) \le \frac{1}{TR\left( \eta \varphi -\frac{\rho \delta h\left( T \right)}{T\varepsilon ^2} \right)},$$
where $h\left( T \right) \triangleq \frac{1}{\beta}( \left( \eta \beta +1 \right) ^T-1 ) -\eta T$. When $\eta \le \frac{1}{\beta}$, the optimality gap $G=\mathcal{L}( \omega ^{(f)} ) -\mathcal{L}\left( \omega ^* \right)$ is bounded by
\begin{align}
\nonumber
G&\le \frac{1}{2\eta \varphi TR}+\rho \delta h\left( T \right) +\sqrt{\frac{1}{4\eta ^2\varphi ^2T^2R^2}+\frac{\rho \delta h\left( T \right)}{\eta \varphi T}} \\
\nonumber
&\le \frac{1}{\eta \varphi TR}+\rho \delta h\left( T \right)+\sqrt{\frac{\rho\delta h(T)}{\eta\varphi T}}.
\end{align}
\end{proof}

Since \SAMPLERNAME~forces FL super nodes to have aligned class distributions, \NAME~has an upper bound of the gradient divergence smaller than FedAvg $\delta_{\mathrm{\NAME}}<\delta_{\mathrm{FedAvg}}$. Therefore, it is easy to infer from Proposition \ref{prop:conv-bound} that, \NAME~has both the convergence upper bound and the optimality gap smaller than those of FedAvg, thus it can improve the FL convergence speed and accuracy performance.

\subsection{Time-Efficiency Condition}
We analyze the time cost of \NAME~in each round ($T$ one-step synchronization iterations) and that of FedAvg in each round ($T$ local iterations), and give the condition under which \NAME~can achieve higher time efficiency than FedAvg.

\textbf{Time Cost of \NAME.}
The time cost of \NAME~$T_{\mathrm{\NAME}}$ is determined by communication, computation and client selection. The communication delays are brought by internal synchronizations $T_{\mathrm{comm}}^{\mathrm{int}}$ and external synchronizations $T_{\mathrm{comm}}^{\mathrm{ext}}$. For the internal synchronization, the delay of uploading local models of size $S$ from $L$ devices to their BS is  $\frac{SL}{B_{\mathrm{up}}^{\mathrm{int}}\log _2\left( 1+\gamma _{\mathrm{BS}} \right)}$, and the delay of synchronizing the model of size $S$ from the BS to $L$ devices is $\frac{SL}{B_{\mathrm{down}}^{\mathrm{int}}\log _2\left( 1+\gamma _{\mathrm{device}} \right)}$, where $B_{\mathrm{up}}^{\mathrm{int}}$ and $B_{\mathrm{down}}^{\mathrm{int}}$ are the uplink and downlink bandwidths between devices and BS, $\gamma _{\mathrm{BS}}$ and $\gamma _{\mathrm{device}}$ are the received signal-to-noise ratio (SNR) of BS and devices. For the external synchronization, the delay of uploading models of size $S$ from $M$ BSs to the top server is $\frac{SM}{B_{\mathrm{up}}^{\mathrm{ext}}\log _2\left( 1+\gamma _{\mathrm{top}} \right)}$, and the delay of synchronizing the global model of size $S$ from the top server to $M$ BSs is $\frac{SM}{B_{\mathrm{down}}^{\mathrm{ext}}\log _2\left( 1+\gamma _{\mathrm{BS}} \right)}$, where $B_{\mathrm{up}}^{\mathrm{ext}}$ and $B_{\mathrm{down}}^{\mathrm{ext}}$ are the uplink and downlink bandwidths between BSs and the top server, $\gamma _{\mathrm{top}}$ is the SNR of the top server. Hence, we have the communication time cost for each internal synchronization and each external synchronization as follows,
\begin{align}
\label{eq:t_ext_comm}
T_{\mathrm{comm}}^{\mathrm{ext}}&=\frac{SM}{B_{\mathrm{up}}^{\mathrm{ext}}\log _2\left( 1+\gamma _{\mathrm{top}} \right)}\,\,+\frac{SM}{B_{\mathrm{down}}^{\mathrm{ext}}\log _2\left( 1+\gamma _{\mathrm{BS}} \right)}, \\
\label{eq:t_int_comm}
T_{\mathrm{comm}}^{\mathrm{int}}&=\frac{SL}{B_{\mathrm{up}}^{\mathrm{int}}\log _2\left( 1+\gamma _{\mathrm{BS}} \right)}+\frac{SL}{B_{\mathrm{down}}^{\mathrm{int}}\log _2\left( 1+\gamma _{\mathrm{device}} \right)}.
\end{align}

Let the delay of each local update be $T_{\mathrm{comp}}$ and the delay of client selection procedure be $T_{\mathrm{select}}$. Each round the internal synchronization is performed $T$ times, the total delay is
\begin{equation}
\label{eq:t_fedgs}
T_{\mathrm{\NAME}}=T_{\mathrm{comm}}^{\mathrm{ext}}+T\cdot\left( T_{\mathrm{select}}+T_{\mathrm{comm}}^{\mathrm{int}}+T_{\mathrm{comp}} \right). 
\end{equation}

\textbf{Time Cost of FedAvg.} The time cost of FedAvg $T_\mathrm{FedAvg}$ is mainly determined by communication and computation because the client selection procedure is simple random sampling so that the selection delay is negligible. The communication delays come from the uplink and downlink model transmission between the top server and devices. The delay of uploading local models of size $S$ from $ML$ devices to the top server is $\frac{SML}{B_{\mathrm{up}}^{\mathrm{ext}}\log _2\left( 1+\gamma _{\mathrm{top}} \right)}$, and the delay of synchronizing the global model of size $S$ from the top server to $ML$ devices is $\frac{SML}{B_{\mathrm{down}}^{\mathrm{ext}}\log _2\left( 1+\gamma _{\mathrm{device}} \right)}$. Hence, we have the communication time cost for each round of synchronization as follows,
\begin{equation}
\label{eq:t_ext_comm_fedavg}
\tilde{T}_{\mathrm{comm}}^{\mathrm{ext}}=\frac{SML}{B_{\mathrm{up}}^{\mathrm{ext}}\log _2\left( 1+\gamma _{\mathrm{top}} \right)}+\frac{SML}{B_{\mathrm{down}}^{\mathrm{ext}}\log _2\left( 1+\gamma _{\mathrm{device}} \right)}.
\end{equation}
Then, the total time cost when FedAvg performs $T$ local updates and one round of synchronization is
\begin{equation}
\label{eq:t_fedavg}
T_\mathrm{FedAvg}=\tilde{T}_{\mathrm{comm}}^{\mathrm{ext}}+T\cdot T_{\mathrm{comp}}.
\end{equation}

In order to simplify the analysis result, we make the following assumptions.

\begin{assumption}
\label{assumption:condition}
(a) The uplink and downlink bandwidths are equal: $B_{\mathrm{up}}^{\mathrm{ext}}=B_{\mathrm{down}}^{\mathrm{ext}}=B^{\mathrm{ext}}$, $B_{\mathrm{up}}^{\mathrm{int}}=B_{\mathrm{down}}^{\mathrm{int}}=B^{\mathrm{int}}$; (b) The SNRs of the top server, BSs, and devices are equal: $\gamma _{\mathrm{top}}=\gamma _{\mathrm{BS}}=\gamma _{\mathrm{device}}=\gamma$.
\end{assumption}

\new{Then, we can give the following condition for hyperparameter setting, under which \NAME~can achieve higher time efficiency than FedAvg.
\begin{proposition}\label{prop:condition}
Let Assumption \ref{assumption:condition} hold and $\beta=\log_2(1+\gamma)$, the time cost per $T$ iterations in each round satisfies $T_{\mathrm{\NAME}} < T_{\mathrm{FedAvg}}$ if $\frac{TL}{M\left( L-1 \right)}<\frac{B^{\mathrm{int}}}{B^{\mathrm{ext}}}$, where
\begin{align}
\label{eq:t_fedgs_simplified}
&T_{\mathrm{\NAME}}=\frac{2SM}{\beta B^{\mathrm{ext}}}+T\left( T_{\mathrm{select}} + \frac{2SL}{\beta B^{\mathrm{int}}} + T_{\mathrm{comp}} \right), \\
\label{eq:t_fedavg_simplified}
&T_{\mathrm{FedAvg}}=\frac{2SML}{\beta B^{\mathrm{ext}}}+T T_{\mathrm{comp}}.
\end{align}
\end{proposition}
}

\begin{proof} Eq. \eqref{eq:t_fedgs_simplified} can be obtained by combining Eqs. \eqref{eq:t_ext_comm}-\eqref{eq:t_fedgs}, and Eq. \eqref{eq:t_fedavg_simplified} can be obtained by combining Eqs. \eqref{eq:t_ext_comm_fedavg}-\eqref{eq:t_fedavg}. Let $T_{\mathrm{\NAME}} - T_{\mathrm{FedAvg}} < 0$, we have
\begin{align}
\nonumber
&\frac{2SM\left( 1-L \right)}{\beta B^{\mathrm{ext}}}+\frac{2TSL}{\beta B^{\mathrm{int}}}+T T_{\mathrm{select}}<0 \\
\nonumber
\Rightarrow\quad&\frac{B^{\mathrm{ext}}}{B^{\mathrm{int}}}SL+T_{\mathrm{select}}\cdot \frac{\beta B^{\mathrm{ext}}}{2}<\frac{SM\left( L-1 \right)}{T}.
\end{align}

In our experiment, \SAMPLERNAME~is quite fast, whose time cost (15 milliseconds) is negligible compared to other delays. Therefore, we assume $T_{\mathrm{select}}\approx 0$ to simplify the result and obtain
\begin{equation}
\nonumber
\frac{B^{\mathrm{ext}}}{B^{\mathrm{int}}}SL<\frac{SM\left( L-1 \right)}{T}\Rightarrow \frac{TL}{M\left( L-1 \right)}<\frac{B^{\mathrm{int}}}{B^{\mathrm{ext}}}.
\end{equation}
\end{proof}

\new{In modern industrial applications, 5G enables indoor industrial use cases that were impossible before, supported by high data rate, ultra-low delay, and extreme density of wireless communications\cite{varga20205g}. In reality, the data rate of 5G edge is about 10-100 times of that in WAN, that is, $B^{\mathrm{int}}/B^{\mathrm{ext}}\in [10,100]$. Therefore, we can easily set $T,M,L$ to satisfy the condition in Proposition \ref{prop:condition}, so that the time efficiency of \NAME~can be guaranteed.}
\begin{figure}[t]
\centering
\includegraphics[width=0.3\textwidth]{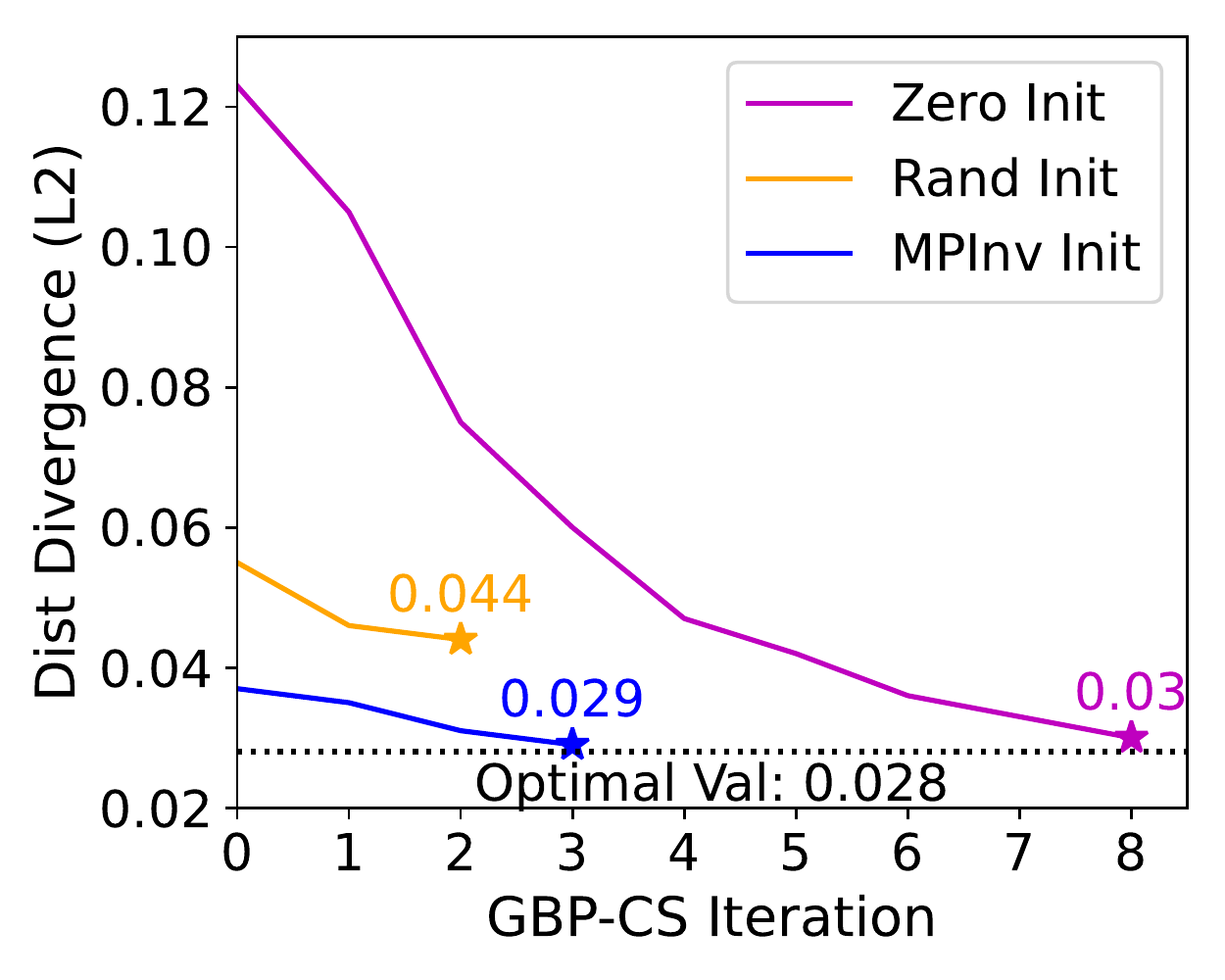}
\caption{Distribution divergence optimization curves of \SAMPLERNAME~with different initializers.}
\label{fig:dist-optim-init-curve}
\end{figure}

\begin{figure*}[t]
\centering
\subfloat[Distribution Divergence]{\includegraphics[width=0.31\textwidth]{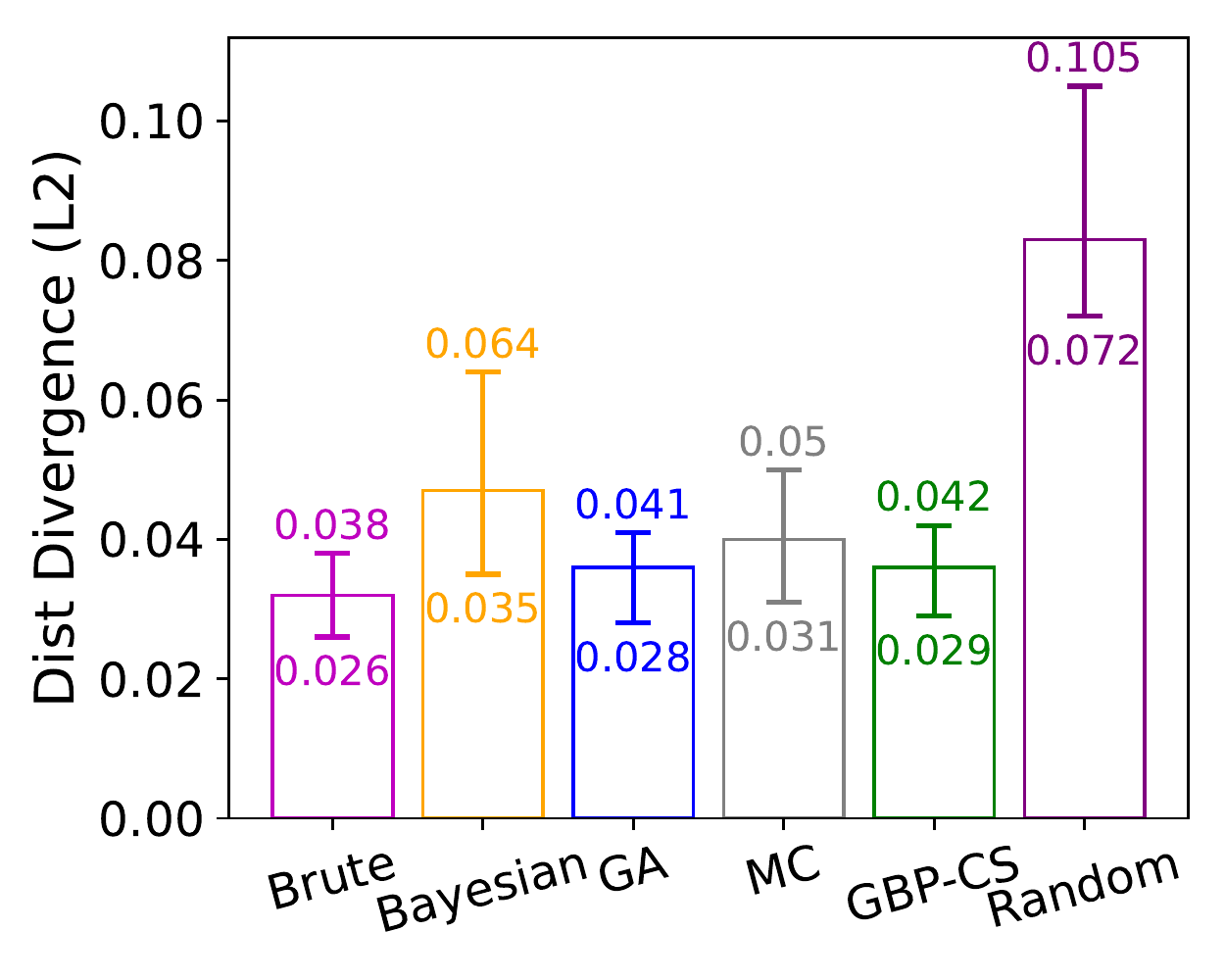}
\label{fig:dist-divergence-comparison}}
\subfloat[Execution Time]{\includegraphics[width=0.33\textwidth]{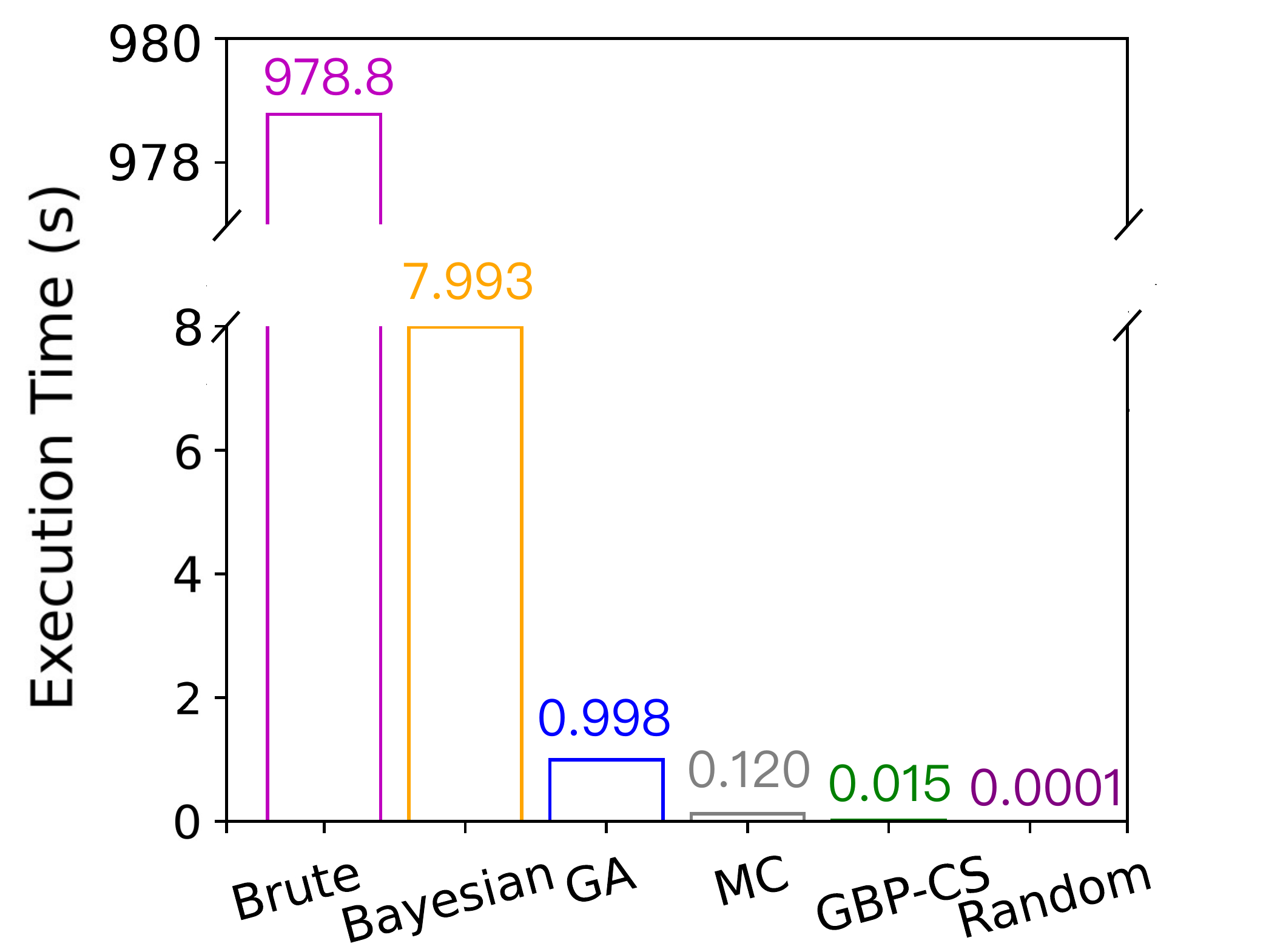}
\label{fig:execution-time-comparison}}
\subfloat[Optimization Curve]{\includegraphics[width=0.325\textwidth]{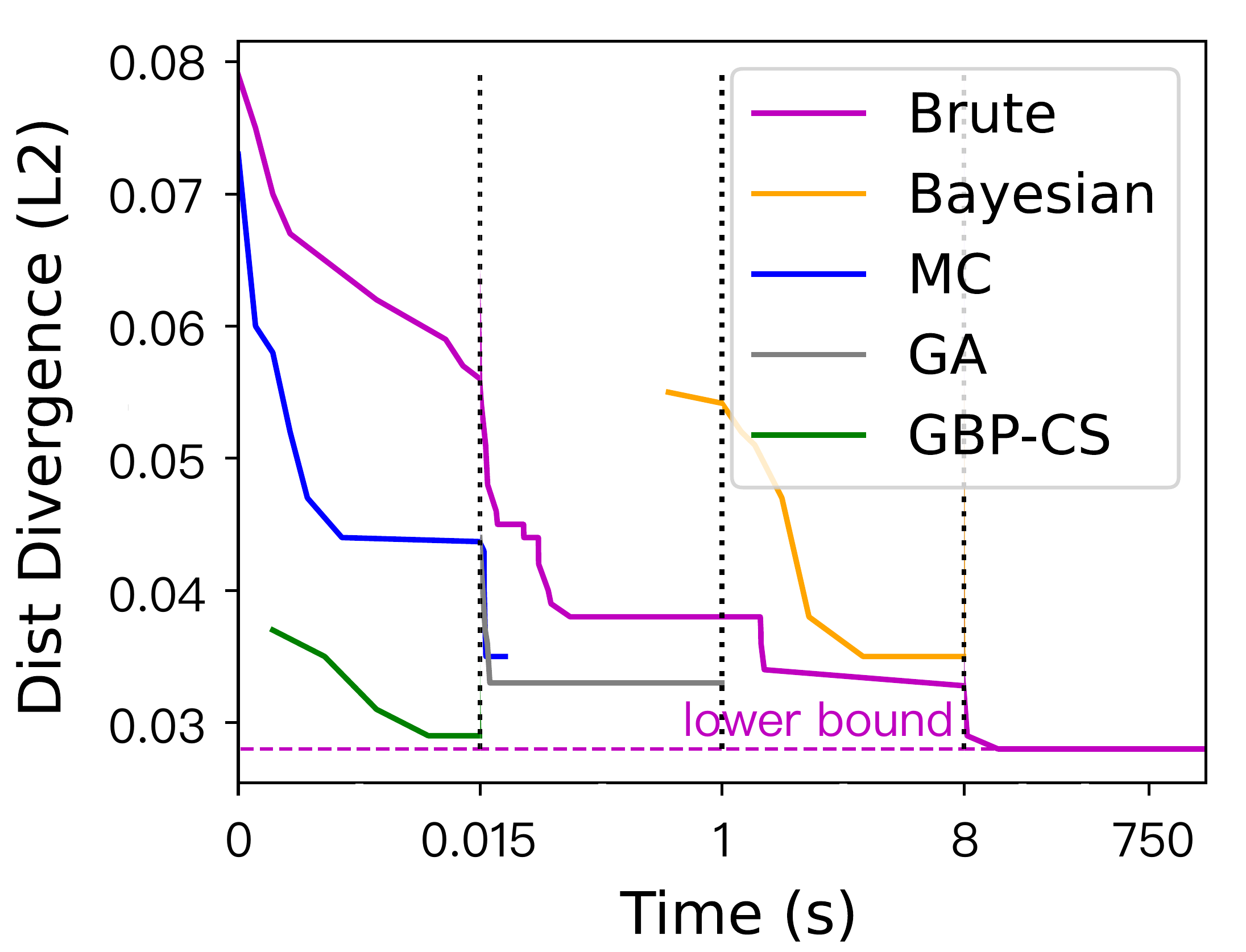}
\label{fig:dist-optim-curve}}
\caption{Comparison of (a) distribution divergence, (b) execution time, and (c) optimization curve among different samplers.}
\label{fig:sampler-comparison}
\end{figure*}

\section{Experimental Evaluation}\label{section:experiment}
\subsection{Experiment Setup}
\textbf{Environment and Hyperparameter Setup.} 
\new{In the experiment, we consider an IIoT application where OCR technology is used in the identification of packing boxes, machines, robots, vehicles, and workers, through recognizing the optical characters on their badges. To this end, we aim to train a high-accuracy OCR model in the federated setting, where sensors' local character images are confidential and skewed in the class distribution. The real-world FEMNIST\cite{caldas2018leaf} dataset is chosen to train our federated OCR model, as it is built by partitioning 805,263 optical digit and character images into 3,550 devices, following non-i.i.d.-like class distributions and uneven data sizes.} Our experiment platform contains $K=350$ OCR cameras and $M=10$ factories, each factory $m$ has  $K^m=35$ OCR cameras (hereinafter referred to as devices). In each iteration, $L=10$ devices are selected from each factory to participate in the federated OCR training. A four-layer convolutional neural network [Conv2D(32), MaxPool, Conv2D(64), MaxPool, Dense(2048), Dense(62)] is used as the training model because it is lightweight and suitable for resource-constrained industrial devices. Unless otherwise specified, we use the standard mini-batch SGD to train local ML models, with the learning rate $\eta=0.01$, the batch size $n=32$, the number of iterations per round $T=50$ and the maximum number of rounds $R=500$.

\textbf{\SAMPLERNAME~Initialization.}
The choice of the initial point $\mathbf{x}_1$ in \SAMPLERNAME~is critical to the quality of the solution, because a bad initial point may cause \SAMPLERNAME~to fall into a local minimum. In the experiment, $L_{\mathrm{rnd}}=2$ devices are pre-sampled at random, and the other $L_{\mathrm{sel}}=L-L_{\mathrm{rnd}}=8$ devices are selected using \SAMPLERNAME, with the following three initialization methods.

\begin{enumerate}
\item \textit{Random Initialization.} Set $L_{\mathrm{sel}}$ values in $\mathbf{x}_1$ to 1 at random and leave other values at 0.
\item \textit{Zero Initialization.} All values in $\mathbf{x}_1$ are first initialized to 0. Then, a warm-up step is performed to meet the vector weight constraint Eq. \eqref{eq:weight-constraint}, in which one value $\mathbf{x}_1\left( i \right)$ with the smallest gradient is set to 1 iteratively until the number of value 1 in $\mathbf{x}_1$ reaches $L_{\mathrm{sel}}$. The warm-up step requires additional $L_{\mathrm{sel}}$ iterations.
\item \textit{Moore-Penrose Inverse Initialization} (MPInv). MPInv is first used to solve the least square solution $\mathbf{\tilde{x}}_1=\mathbf{A}^{-1}\mathbf{y}$ of the unconstrained objective function $\underset{\mathbf{x}}{\min} \left\| \mathbf{Ax}-\mathbf{y} \right\| _{L_2}$. Then, $L_{\mathrm{sel}}$ elements with the largest values in $\mathbf{\tilde{x}}_1$ are set to 1 and others are left at 0 to obtain the initial point $\mathbf{x}_1$.
\end{enumerate}

\textbf{Comparison Algorithms.} 
To highlight the efficiency and effectiveness of the proposed \SAMPLERNAME, we consider the following five benchmark client selection methods for comparison.

\begin{enumerate}
\item \textit{Random Sampler} (Random): From each group, $L_{\mathrm{sel}}$ devices are uniformly and randomly sampled.
\item \textit{Monte Carlo Sampler} (MC): Repeat the random sampler 1000 times and the solution minimizes Eq. \eqref{eq:simple-goal} is used.
\item \textit{Brute Sampler} (Brute): Brutely search for the optimal $L_{\mathrm{sel}}$ devices by traversing all feasible solutions to meet Eqs. \eqref{eq:simple-goal}-\eqref{eq:weight-constraint}.
\item \textit{Bayesian Sampler} (Bayesian): Search for a sub-optimal $L_{\mathrm{sel}}$ devices using Bayesian optimization\cite{nogueira2014bayesian} to meet Eqs. \eqref{eq:simple-goal}-\eqref{eq:weight-constraint}. By default, we set the number of initial points to 5 and exploration iterations to 25.
\item \textit{Genetic Sampler} (GA): Search for a sub-optimal $L_{\mathrm{sel}}$ devices using genetic algorithm\cite{whitley1994genetic} to meet Eqs. \eqref{eq:simple-goal}-\eqref{eq:weight-constraint}, in which the constrained 0-1 vector solutions are regarded as genes and suffer from selection, crossover, mutation and elimination. By default, we set the population size to 100, the mutation probability to 0.001, and the number of generations to 100.
\end{enumerate}

Except for the baseline FedAvg\cite{mcmahan2016communication}, other nine advanced approaches are also experimentally compared with \NAME~in the presence of non-i.i.d. data. They are FedMMD\cite{yao2018two}, FedFusion\cite{yao2019towards}, FedProx\cite{li2018federated}, IDA\cite{yeganeh2020inverse}, CGAU\cite{rieger2020client}, FedAvgM\cite{hsu2019measuring}, and FedAdagrad, FedAdam, FedYogi from \cite{reddi2020adaptive}.

\textbf{Implementation.} 
We implement \NAME~on a standard FL simulator Leaf-MX\footnote{Leaf-MX: \url{https://github.com/Lizonghang/leaf-mx}} (an MXNET\cite{chen2015mxnet} implementation of LEAF\cite{caldas2018leaf}). 
The code implementation is open-available on Github: \url{https://github.com/Lizonghang/fedgs}.

\subsection{Results and Discussion}
\textbf{Comparison of initialization methods in \SAMPLERNAME.} 
The optimization curves of the class distribution divergence of Zero, Random, and MPInv initializers are shown in Fig. \ref{fig:dist-optim-init-curve}. Both Zero and MPInv initializers successfully find high-quality solutions (0.029 and 0.030, respectively) close to the optimal of the brute force search (0.028). Instead, the Random initializer falls into a poor local optimal (0.044). Furthermore, the MPInv initializer is much faster because it does not require an additional warm-up procedure like the Zero initializer. Therefore, \SAMPLERNAME~is default initialized with MPInv initializer.

\begin{figure}[!t]
\centering
\subfloat[Effect of $n$ and $T$]{\includegraphics[width=0.35\textwidth]{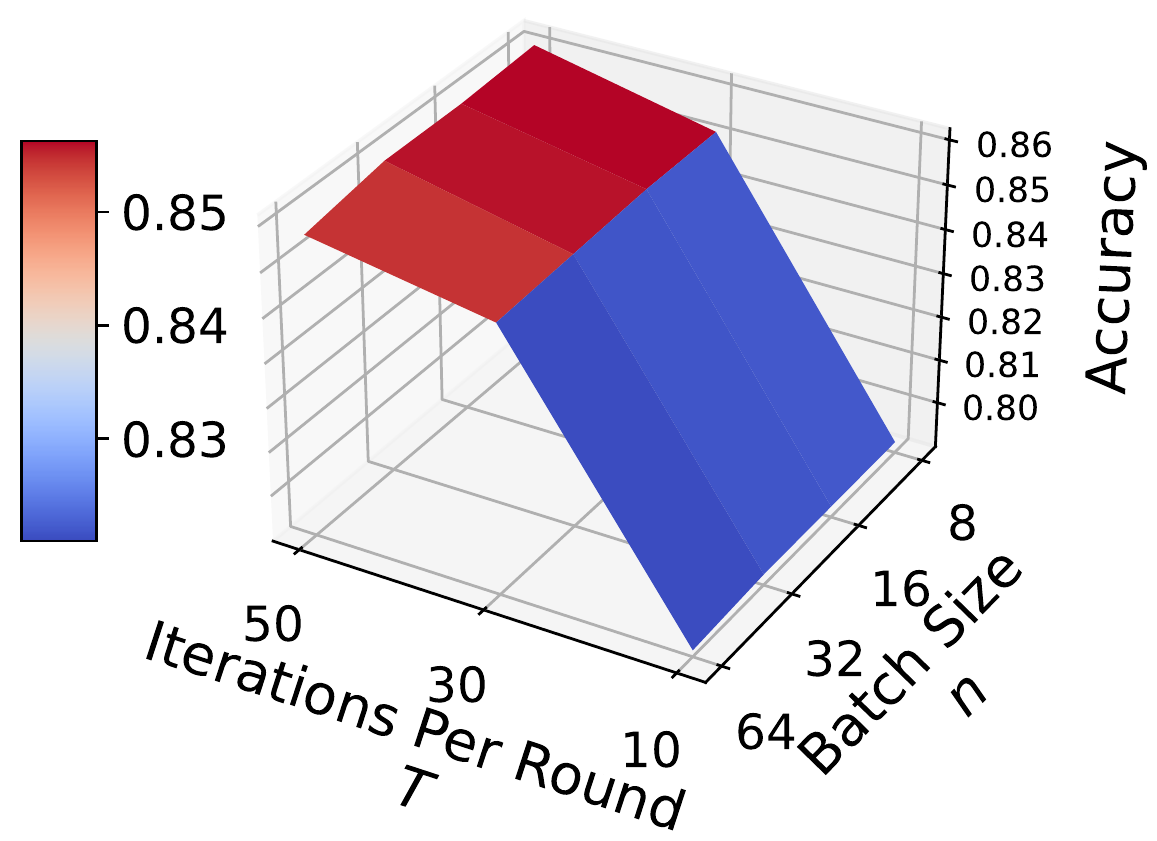}
\label{fig:acc-surface-batch-iter}}
\\
\subfloat[Effect of $M$ and $L$]{\includegraphics[width=0.35\textwidth]{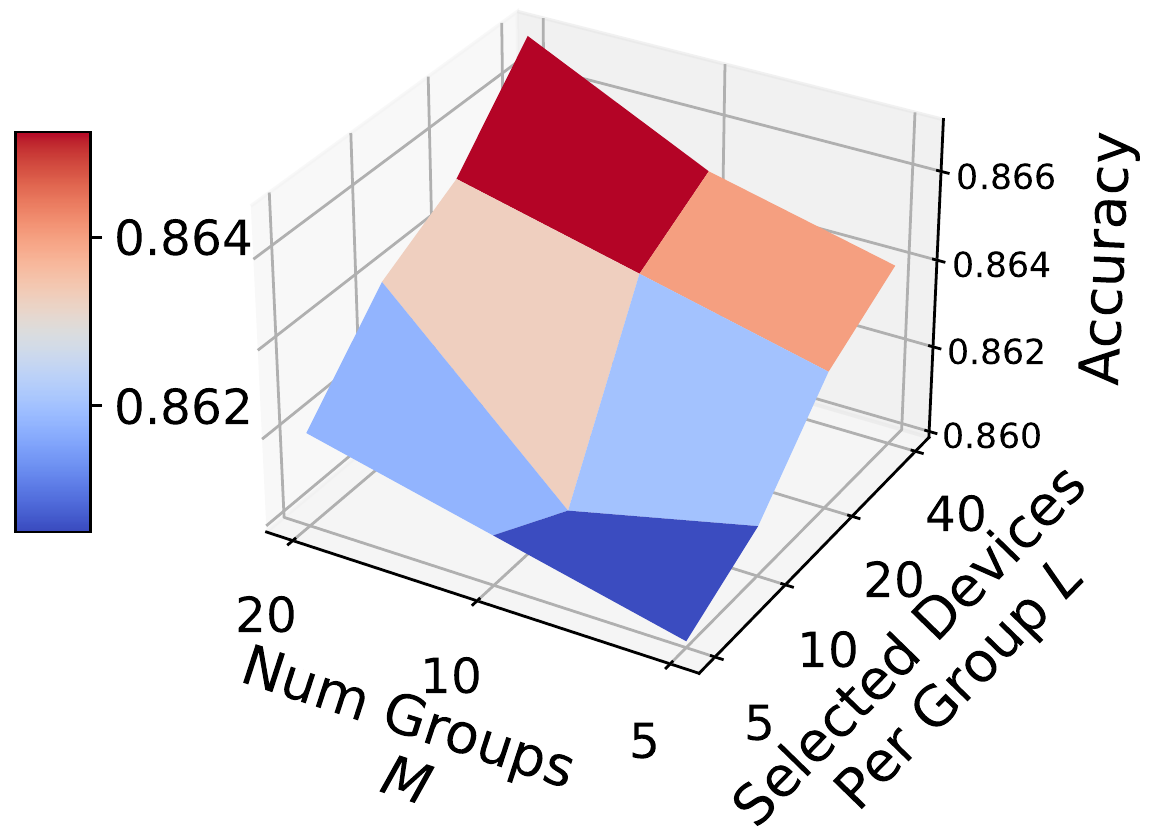}
\label{fig:acc-surface-groups-devices}}
\caption{Accuracy surface of \NAME~over different (a) batch size $n$ and iterations per round $T$; (b) number of groups $M$ and number of selected devices per group $L$.}
\label{fig:acc-surface}
\end{figure}

\textbf{Comparison among \SAMPLERNAME~and other samplers.}
Since the procedure of \SAMPLERNAME~client selection is performed every iteration, both the quality and time cost of the solution are critical to \NAME~performance.

We first compare the distribution divergence (defined as Eq. \eqref{eq:goal}) among \SAMPLERNAME~and other five benchmark samplers. Generally speaking, the smaller the gap between the class distribution $\mathcal{P}^m_t$ of the group $m$ and the global distribution $\mathcal{P}^{\mathrm{real}}$, the smaller the distribution divergence and the better the sampler. The distribution divergence of $M=10$ factories is shown in Fig. \ref{fig:dist-divergence-comparison}. As expected, the most commonly used random sampler in FedAvg leads to a high divergence in class distribution ($0.072\sim0.105$) and causes non-i.i.d. data among groups, while the brute force sampler can always minimize the divergence ($0.026\sim0.038$). The random sampler and the brute force sampler give the upper and lower bounds of the distribution divergence, and solutions of other samplers should locate in this interval, among which GA and \SAMPLERNAME~samplers perform best ($0.028\sim0.041$ and $0.029\sim0.042$, respectively).

In terms of execution time, \NAME~prefers samplers with a very short execution time, because high-frequency client selection may introduce a non-negligible latency to \NAME~and significantly slow down FL training. Fig. \ref{fig:execution-time-comparison} compares the execution time of the above samplers. Let us focus on the brute force, GA, and \SAMPLERNAME~samplers, because their solutions are of the best quality. The brute force sampler requires 979 seconds to find the optimal solution, whose latency is too long to be acceptable. Therefore, \NAME~prefers a sub-optimal solution in an acceptable short time. GA and \SAMPLERNAME~samplers seem to be good choices, and the proposed \SAMPLERNAME~sampler is 66$\times$ faster than the GA sampler, with a negligible 15 milliseconds and a loss of distribution divergence by only 0.001. 

To highlight \SAMPLERNAME~more intuitively, we draw the optimization curve of distribution divergence over execution time in Fig. \ref{fig:dist-optim-curve}. The results show that the proposed \SAMPLERNAME~sampler converges to a high-quality solution 0.029 closest to the optimal 0.028 in the shortest time, demonstrating the superior effectiveness and efficiency of \SAMPLERNAME.

\textbf{Effects of hyperparameters in FedGS.}
Hyperparameters may have great effects on \NAME. To explore these effects, we perform a grid search on experimental hyperparameters, including the batch size $n$, the number of iterations per round $T$, the number of devices selected per group $L$, as well as the environmental hyperparameter, the number of groups $M$. Fig. \ref{fig:acc-surface-batch-iter} visualizes the test accuracy over different $n$ and $T$ settings, where $n$ is chosen from $\{8, 16, 32, 64\}$ and $T$ is chosen from $\{10, 30, 50\}$. The results show that a moderately large $T$ can improve the accuracy of \NAME, while the batch size $n$ has little effect. Fig. \ref{fig:acc-surface-groups-devices} visualizes the test accuracy over different $M$ and $L$ settings, where $M$ is chosen from $\{5, 10, 20\}$ and $L$ is chosen from $\{5, 10, 20, 40\}$. Without loss of generality, both more groups and more selected devices can bring gains in FL model accuracy because more devices' data is included. In this paper, $n=32$, $T=50$ and $L=10$ are used by default to meet the condition in Proposition \ref{prop:condition}. Please note that $M=10$ is determined by the real-world environment instead of an adjustable hyperparameter.

\begin{table}[t]
  \caption{Test accuracy, test loss and convergence speed of \NAME~vs ten federated approaches.}
  \label{table:comparison}
  \centering
  \renewcommand{\arraystretch}{1.3}
  \begin{tabular}{c|c|c|c}
    \hline
    & Test Accuracy & Test Loss & Rounds To 82\% \\
    \hline
    FedAvg (Baseline) & 82.1\% & 0.587 & 478 \\
    \hline
    FedProx & 82.0\% & 0.586 & 497 \\
    \hline
    IDA & 81.0\% & 0.628 & $\times$ \\
    IDA+INTRAC & 81.0\% & 0.618 & $\times$ \\
    IDA+FedAvg & 80.5\% & 0.687 & $\times$ \\
    \hline
    CGAU & 83.3\% & 0.509 & 202 \\
    \hline
    FedMMD & 83.0\% & 0.564 & 378 \\
    \hline
    FedFusion+Conv & 81.7\% & 0.624 & $\times$ \\
    FedFusion+Multi & 82.0\% & 0.591 & 486 \\
    FedFusion+Single & 80.7\% & 0.627 & $\times$ \\
    \hline
    FedAvgM & 84.4\% & 0.820 & \textbf{68} \\
    \hline
    FedAdagrad & 83.8\% & 0.583 & 264 \\
    \hline
    FedAdam & 85.0\% & 0.662 & 71 \\
    \hline
    FedYogi & 84.6\% & 0.590 & 76 \\
    \hline\hline
    \textbf{\NAME} & \textbf{86.0\%} & \textbf{0.435} & 147 \\
    \hline
  \end{tabular}
\end{table}

\textbf{Comparison among \NAME~and other federated approaches.}
We take ten advanced federated approaches for comparison to show the state-of-the-art performance of the proposed \NAME~in the presence of non-i.i.d. data. The test accuracy, test loss, and training rounds required to reach the accuracy of 82\% are listed in Table \ref{table:comparison}, and detailed training curves are given in Fig. \ref{fig:comparison-alg}. Unless otherwise specified, all the comparison approaches use the local epoch $e=5$ by default. In the following, we will compare these approaches and analyze their results, respectively.

\begin{figure*}[!t]
\centering
\subfloat[\NAME~vs FedProx]{\includegraphics[width=0.34\textwidth]{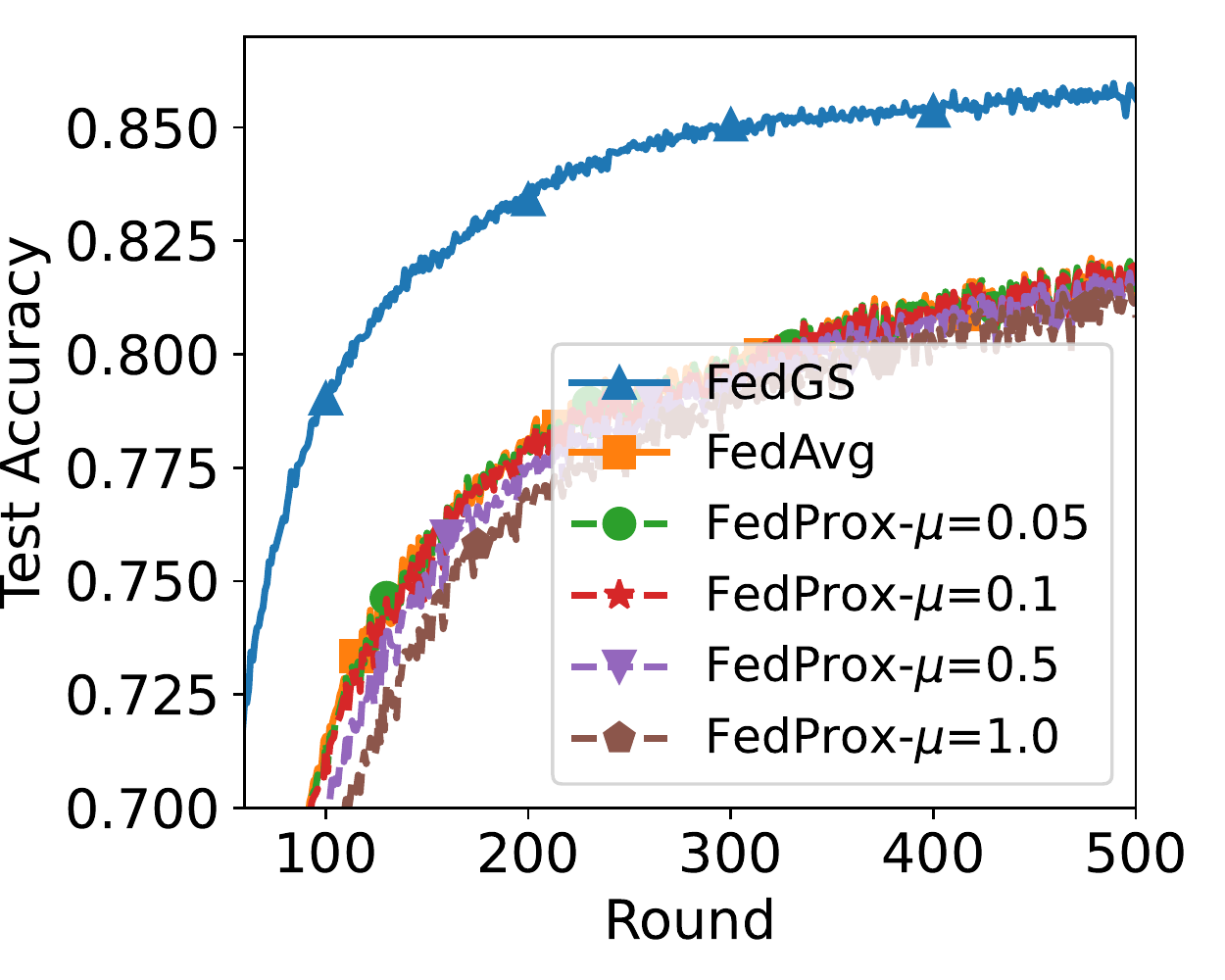}
\label{fig:acc-vs-fedprox}}
\subfloat[\NAME~vs IDA]{\includegraphics[width=0.34\textwidth]{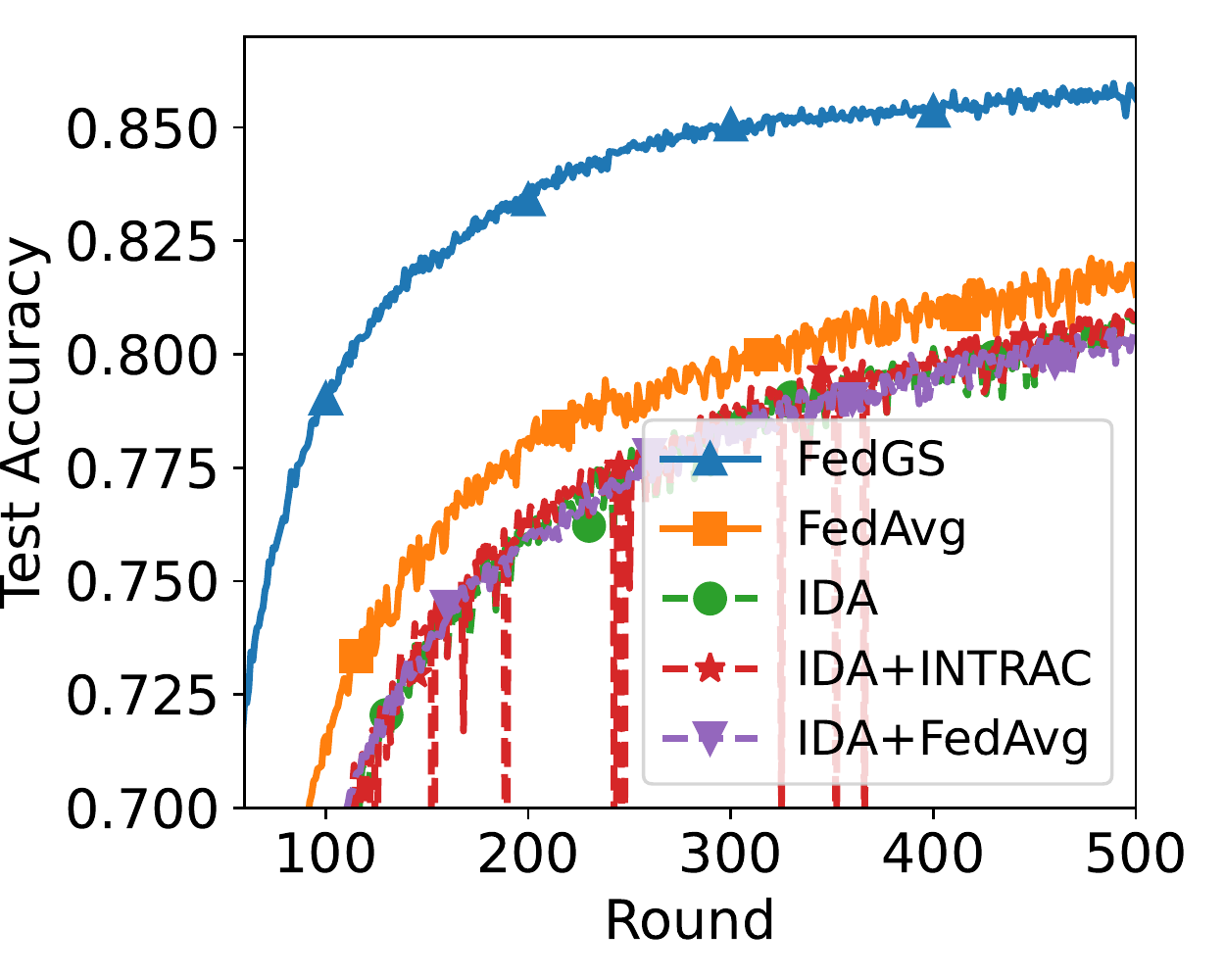}
\label{fig:acc-vs-ida}}
\subfloat[\NAME~vs CGAU]{\includegraphics[width=0.34\textwidth]{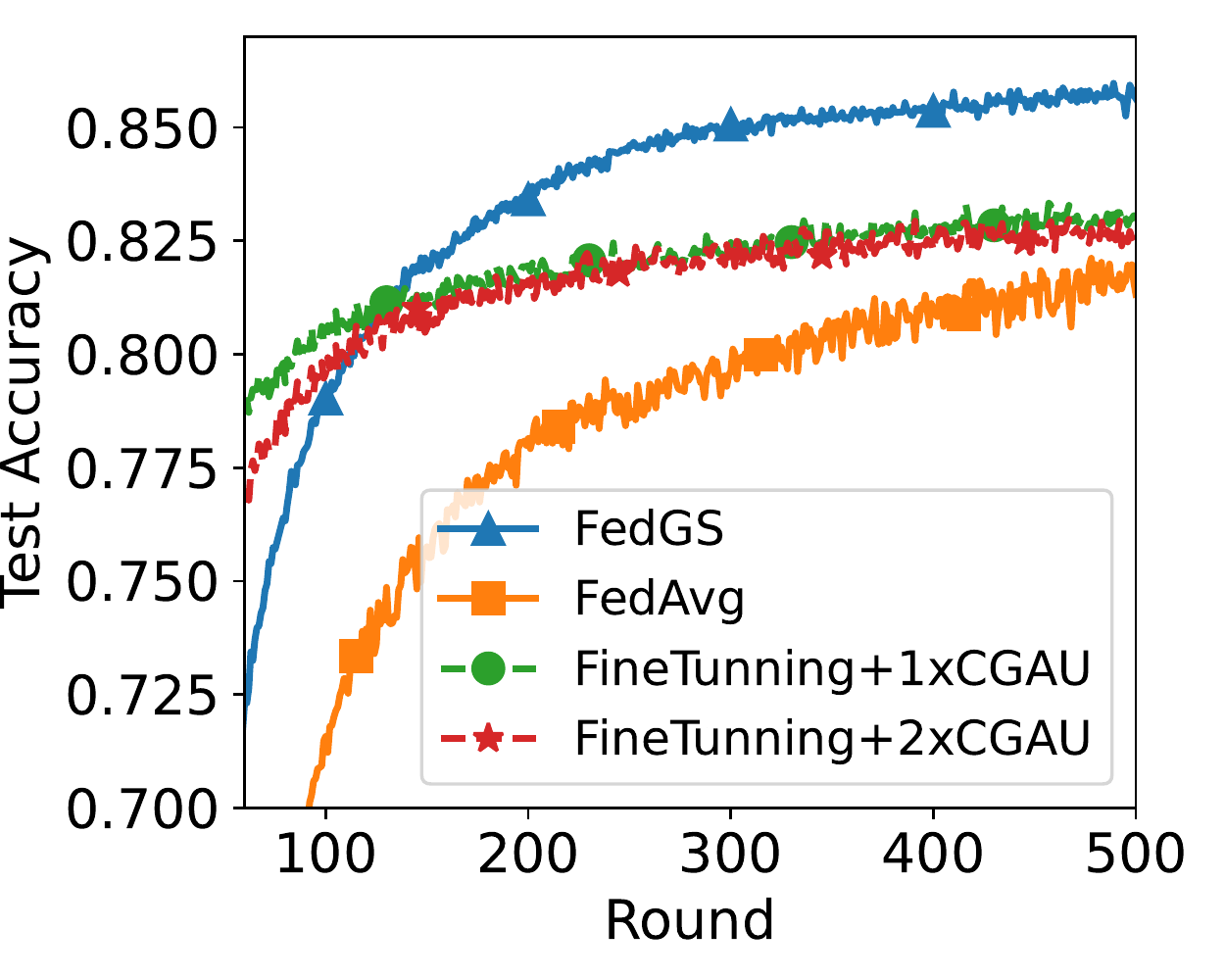}
\label{fig:acc-vs-cgau}}
\\
\subfloat[\NAME~vs FedProx]{\includegraphics[width=0.34\textwidth]{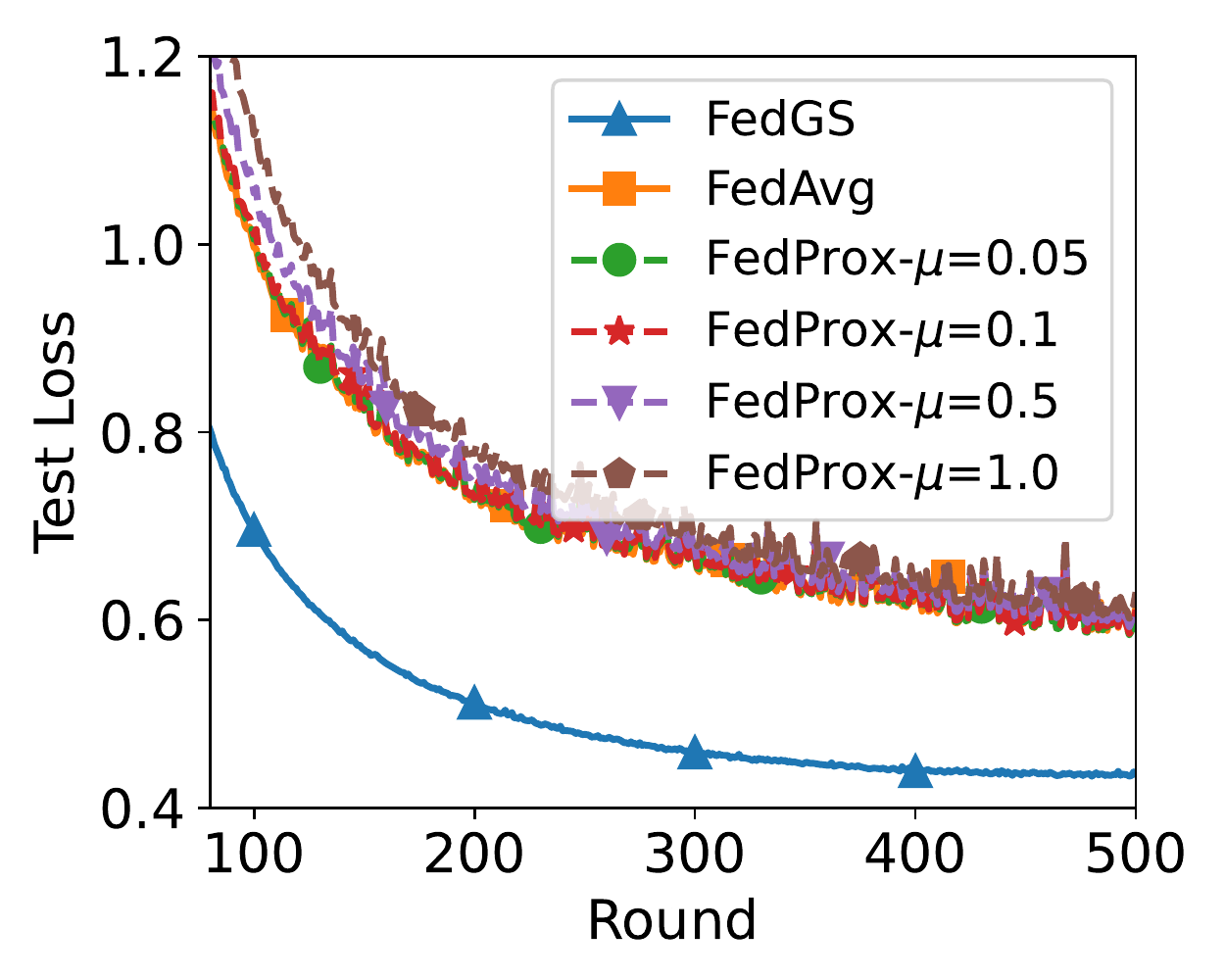}
\label{fig:loss-vs-fedprox}}
\subfloat[\NAME~vs IDA]{\includegraphics[width=0.34\textwidth]{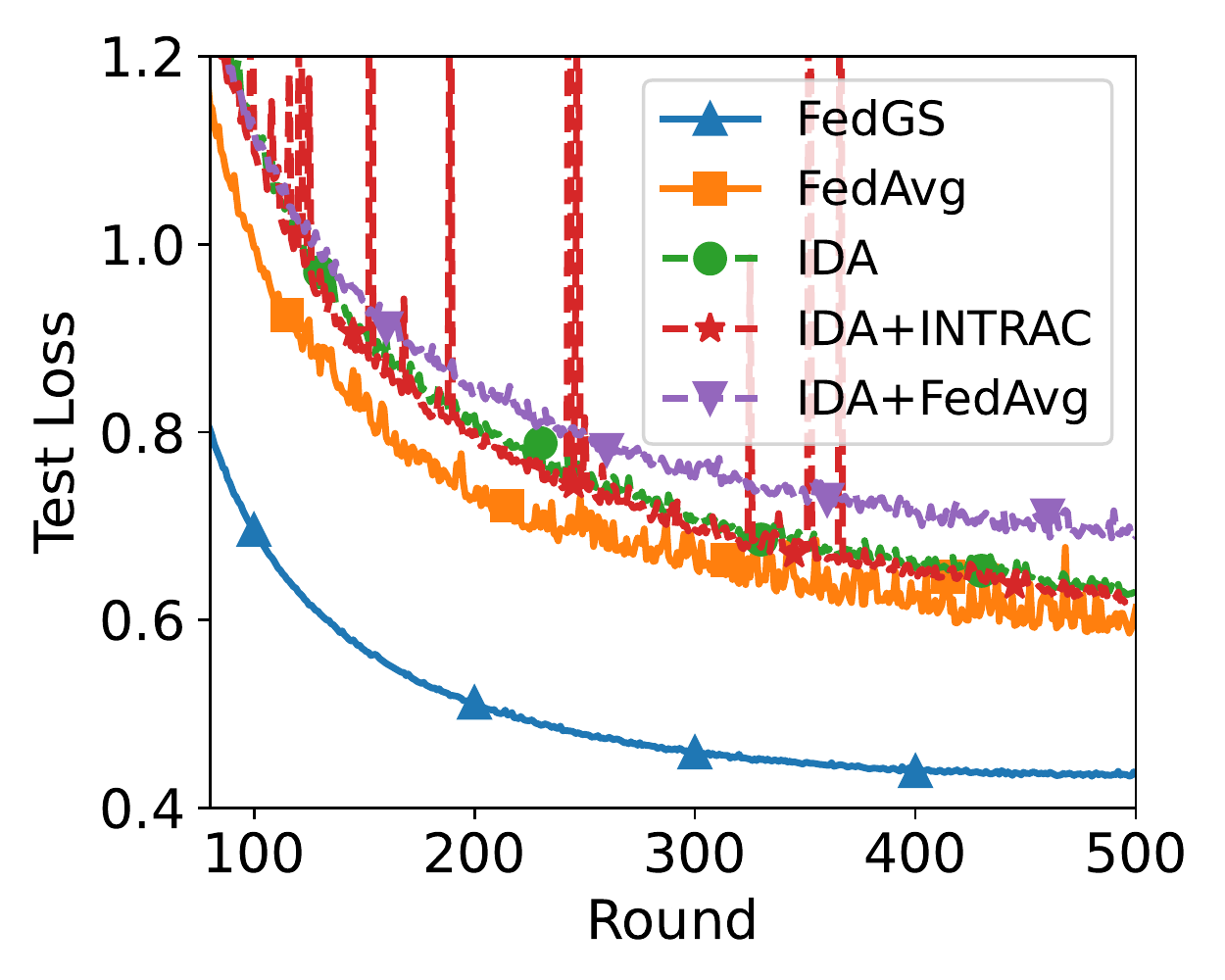}
\label{fig:loss-vs-ida}}
\subfloat[\NAME~vs CGAU]{\includegraphics[width=0.34\textwidth]{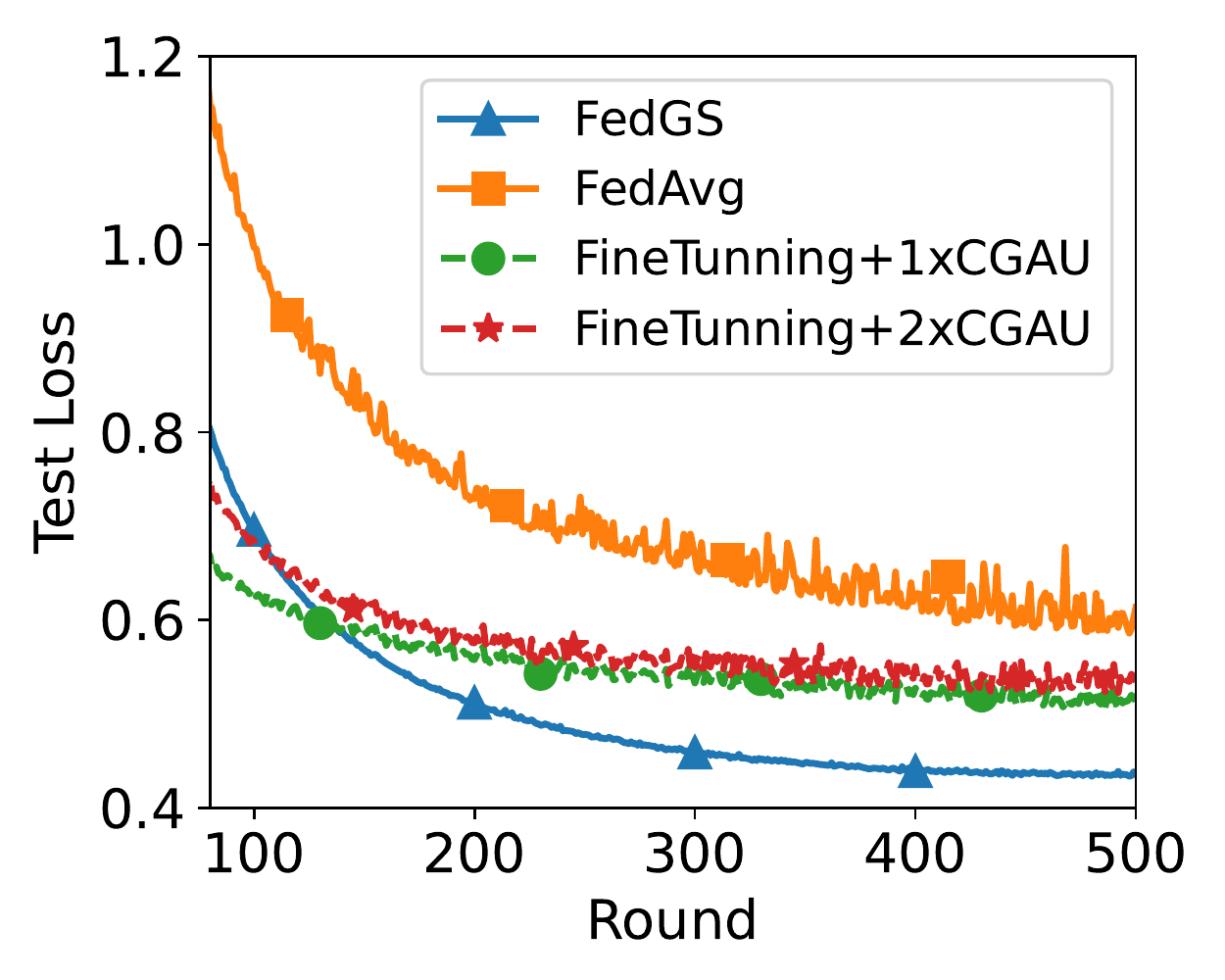}
\label{fig:loss-vs-cgau}}
\\
\subfloat[\NAME~vs FedMMD]{\includegraphics[width=0.34\textwidth]{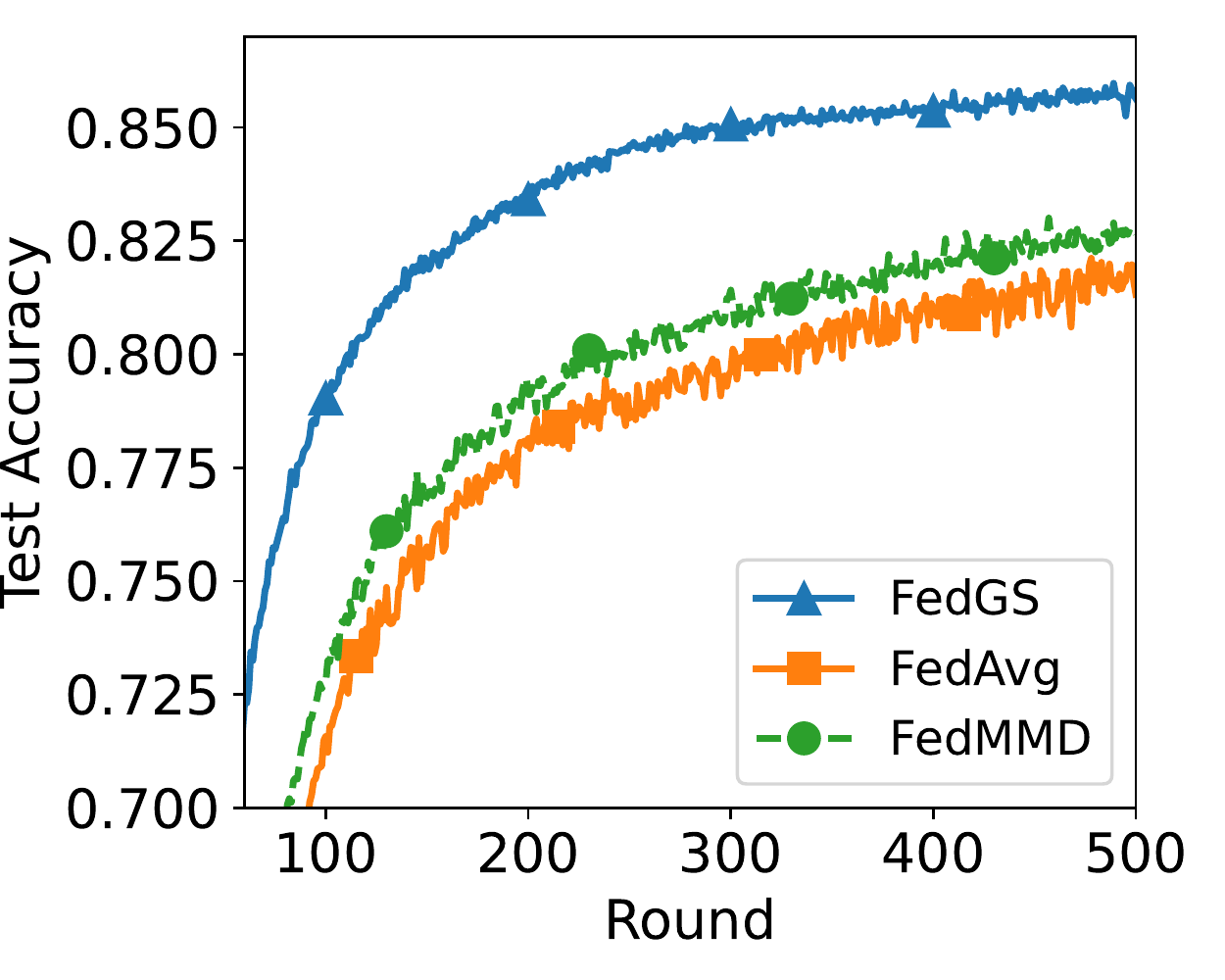}
\label{fig:acc-vs-fedmmd}}
\subfloat[\NAME~vs FedFusion]{\includegraphics[width=0.34\textwidth]{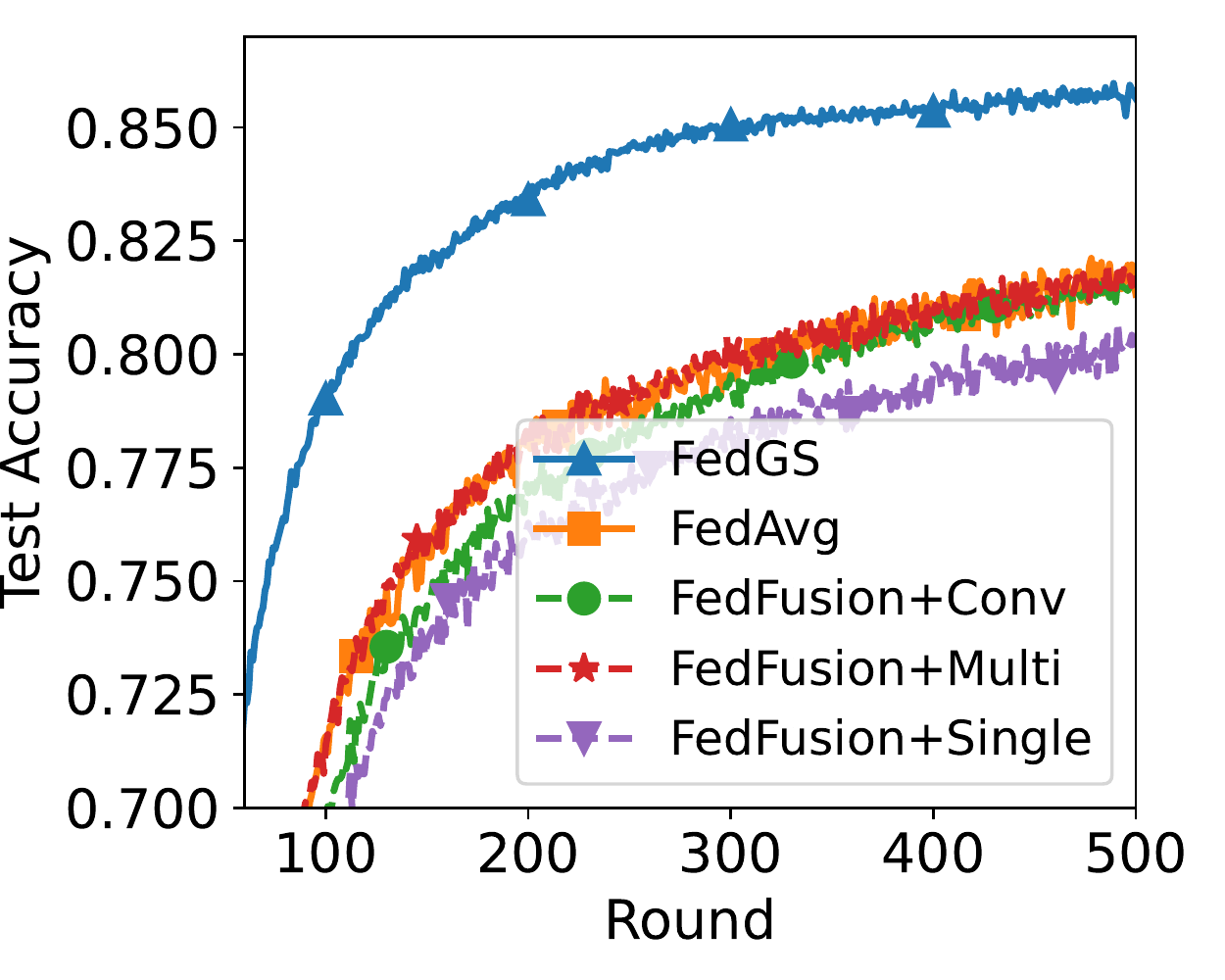}
\label{fig:acc-vs-fedfusion}}
\subfloat[\NAME~vs FedOpt]{\includegraphics[width=0.34\textwidth]{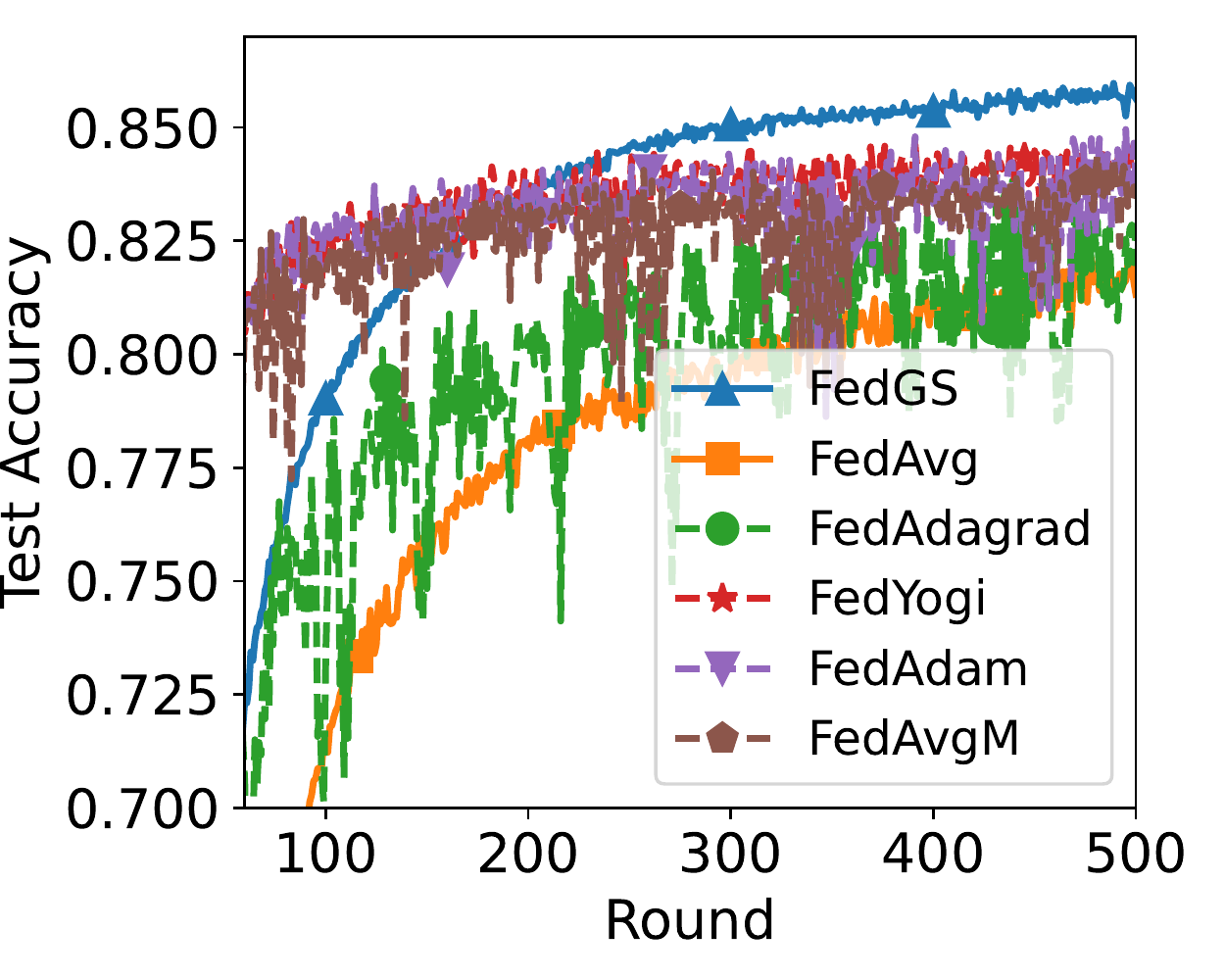}
\label{fig:acc-vs-fedopt}}
\\
\subfloat[\NAME~vs FedMMD]{\includegraphics[width=0.34\textwidth]{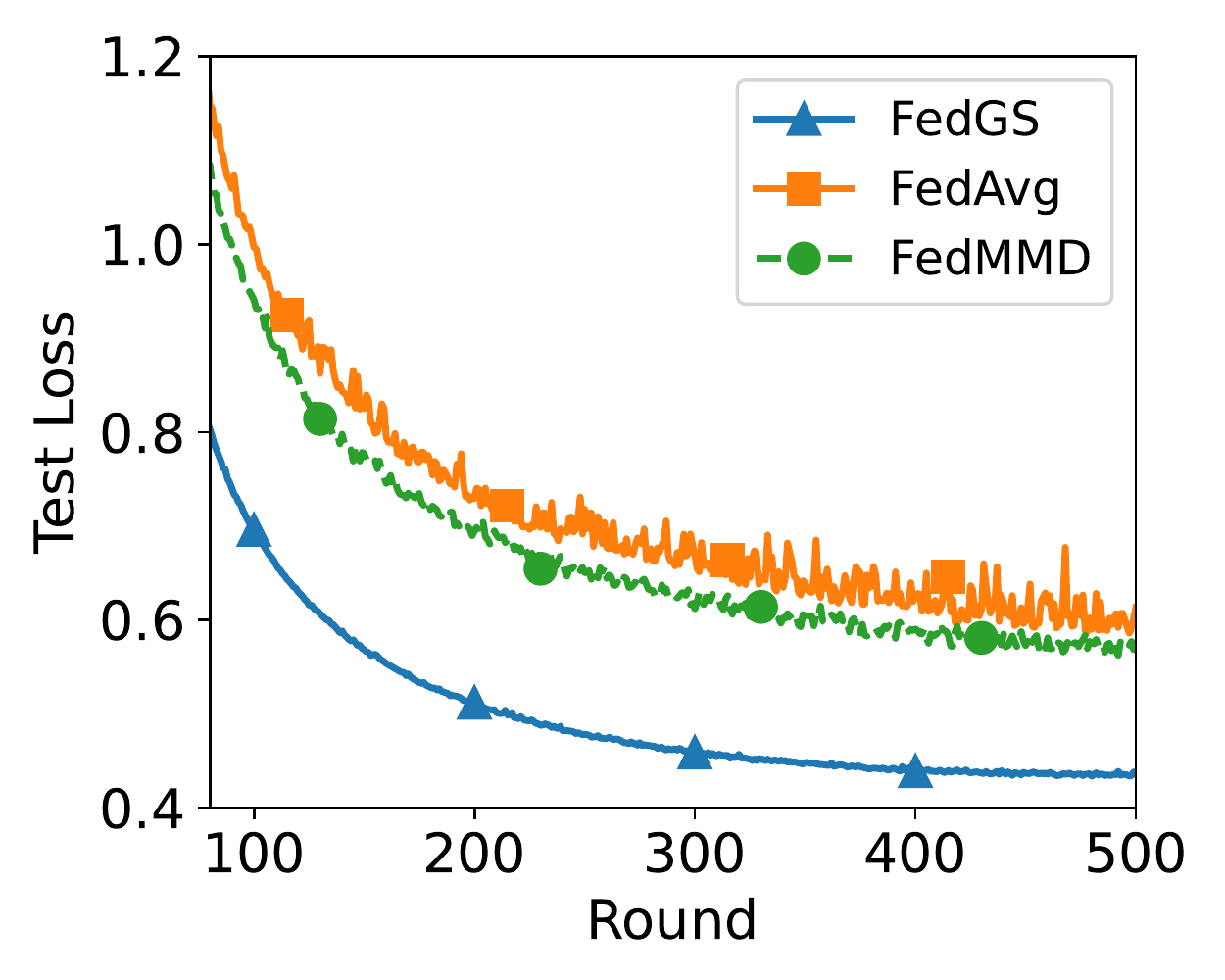}
\label{fig:loss-vs-fedmmd}}
\subfloat[\NAME~vs FedFusion]{\includegraphics[width=0.34\textwidth]{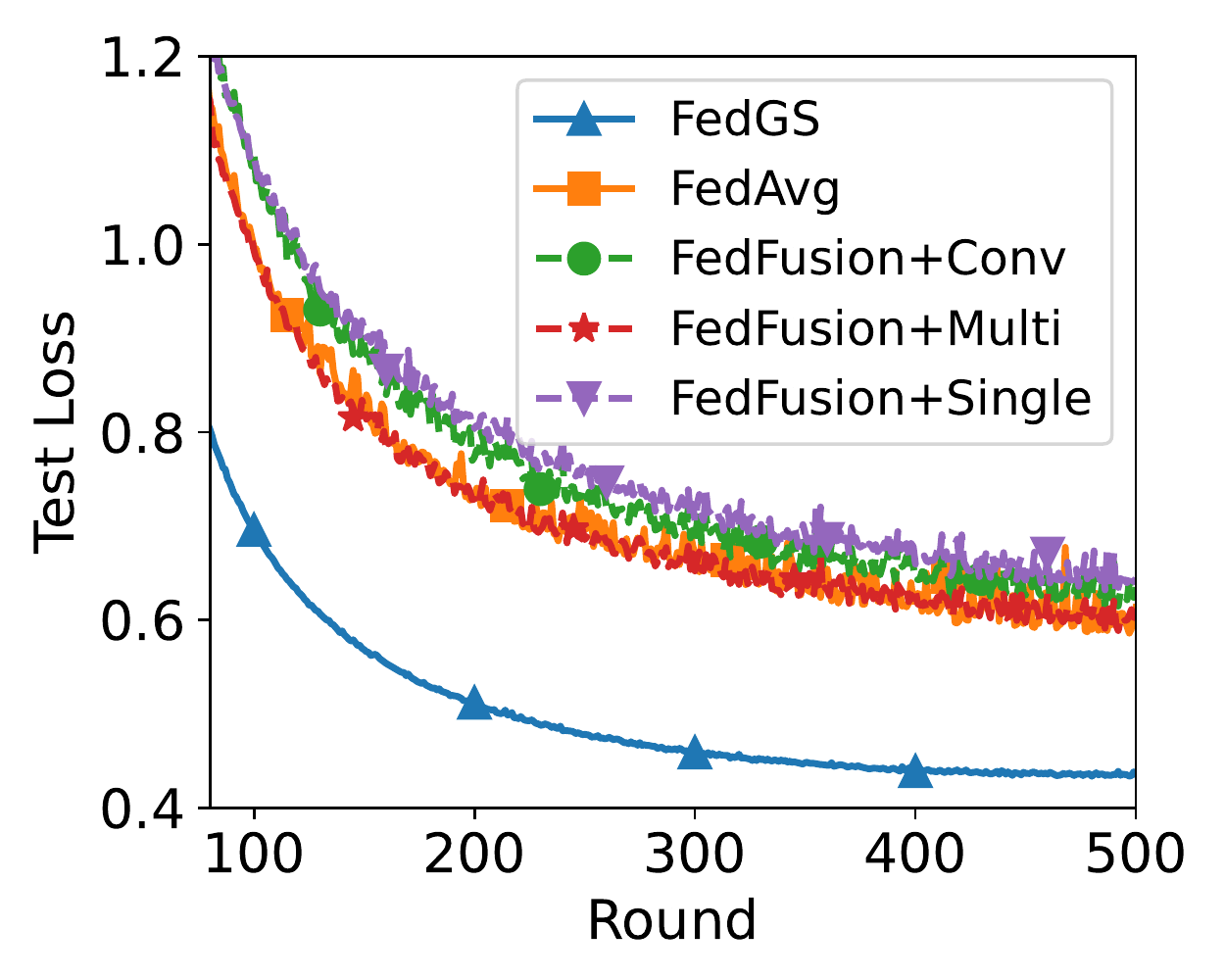}
\label{fig:loss-vs-fedfusion}}
\subfloat[\NAME~vs FedOpt]{\includegraphics[width=0.34\textwidth]{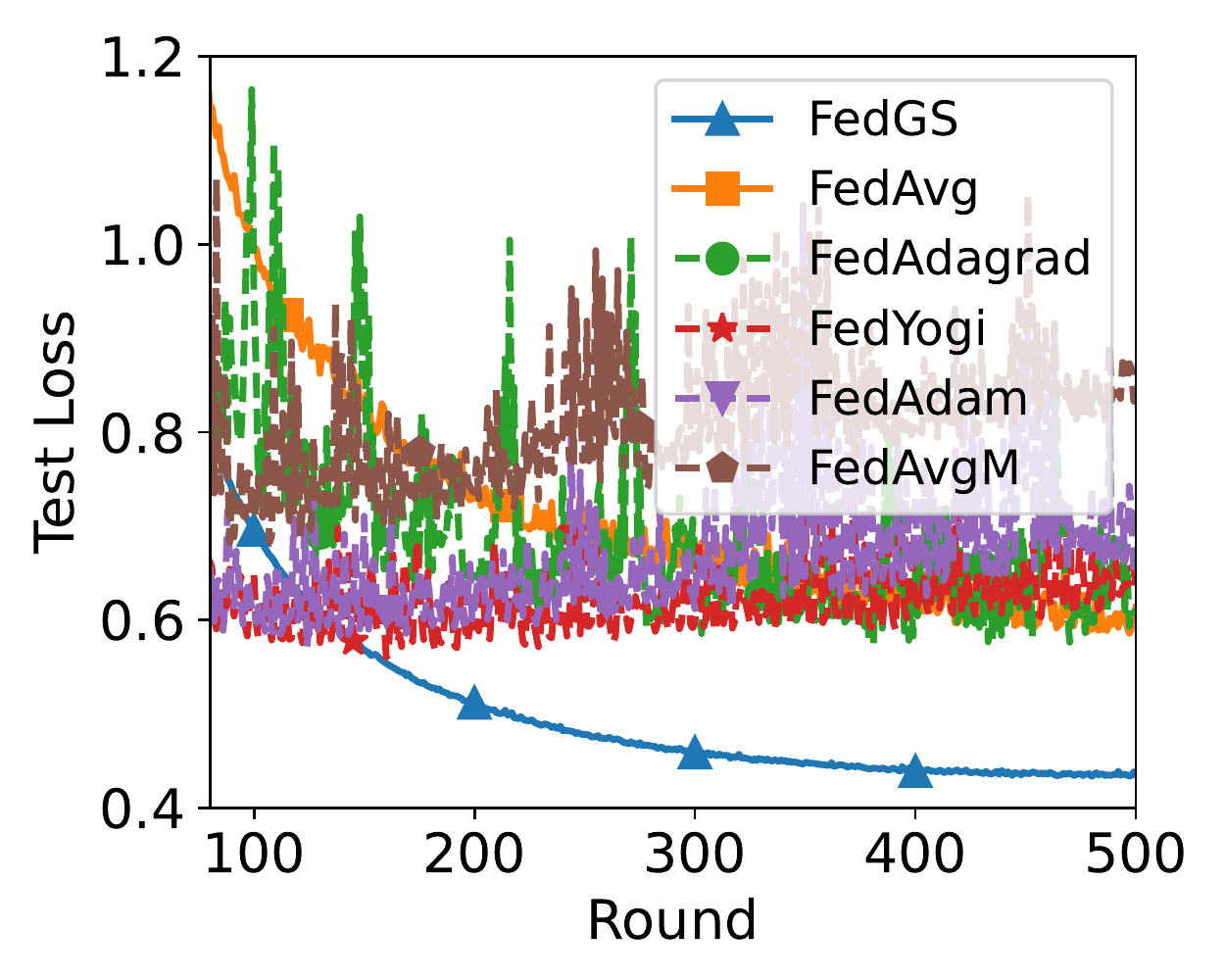}
\label{fig:loss-vs-fedopt}}
\caption{Comparison of \NAME~and FedAvg, FedProx, IDA, CGAU, FedMMD, FedFusion, FedAvgM, FedAdagrad, FedAdam, FedYogi.}
\label{fig:comparison-alg}
\end{figure*}

\textit{\NAME~vs FedProx.} FedProx adds a proximal term to local loss functions to penalize divergent local models. We tune the penalty constant $\mu=\{0.05,0.1,0.5,1.0\}$ to find the best result in Figs. \ref{fig:acc-vs-fedprox} and \ref{fig:loss-vs-fedprox}. However, FedProx performs poorly in our case, with the accuracy of 82.0\%, not even exceeding the baseline accuracy of 82.1\% of FedAvg. The reason may be that the proximal penalty term will slow convergence by forcing local models closer to the starting point\cite{li2018federated}. Instead, the proposed \NAME~improves the baseline accuracy by 3.9\% and achieves the accuracy of 86.0\%.

\textit{\NAME~vs IDA.} IDA weighs model parameters of devices based on their inverse distance to the averaged model parameter during aggregation. We combine IDA with inverse training accuracy coefficients (IDA+INTRAC) and normalized data size coefficients (IDA+FedAvg) as suggested by the authors. However, Figs. \ref{fig:acc-vs-ida} and \ref{fig:loss-vs-ida} shows that IDA-series approaches suffer an accuracy degradation ($80.5\%\sim81.0\%$). That is because devices with large parameter deviations are over-suppressed, causing the global model to lose data knowledge on these devices. Besides, IDA should cache model parameters uploaded by all devices until the average model parameter and inverse distance coefficients are calculated, which takes up huge memory space on the server.

\begin{figure}[t]
\centering
\subfloat[FedAvgM]{\includegraphics[width=0.24\textwidth]{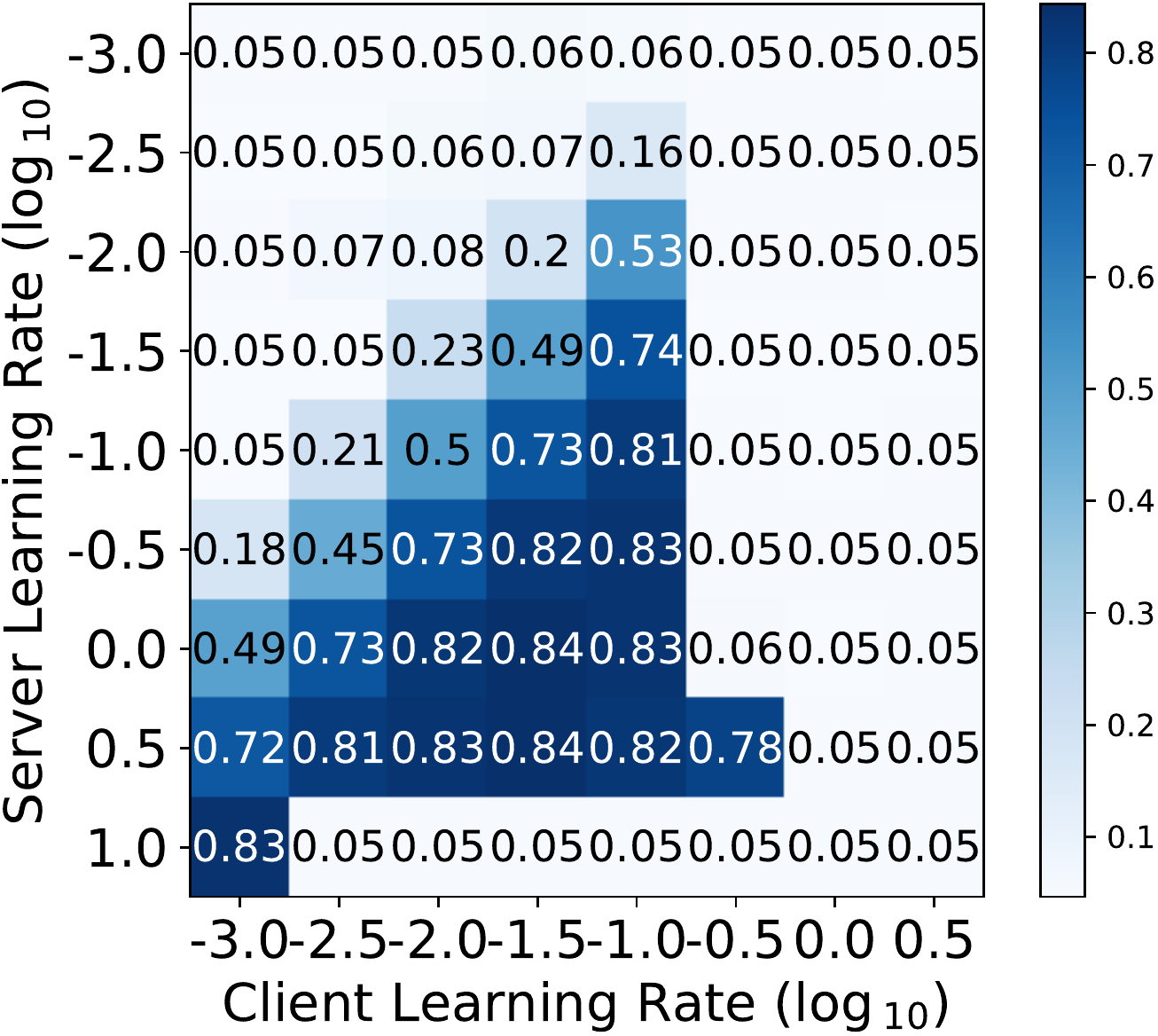}
\label{fig:heatmap-fedavgm}}
\subfloat[FedAdagrad]{\includegraphics[width=0.24\textwidth]{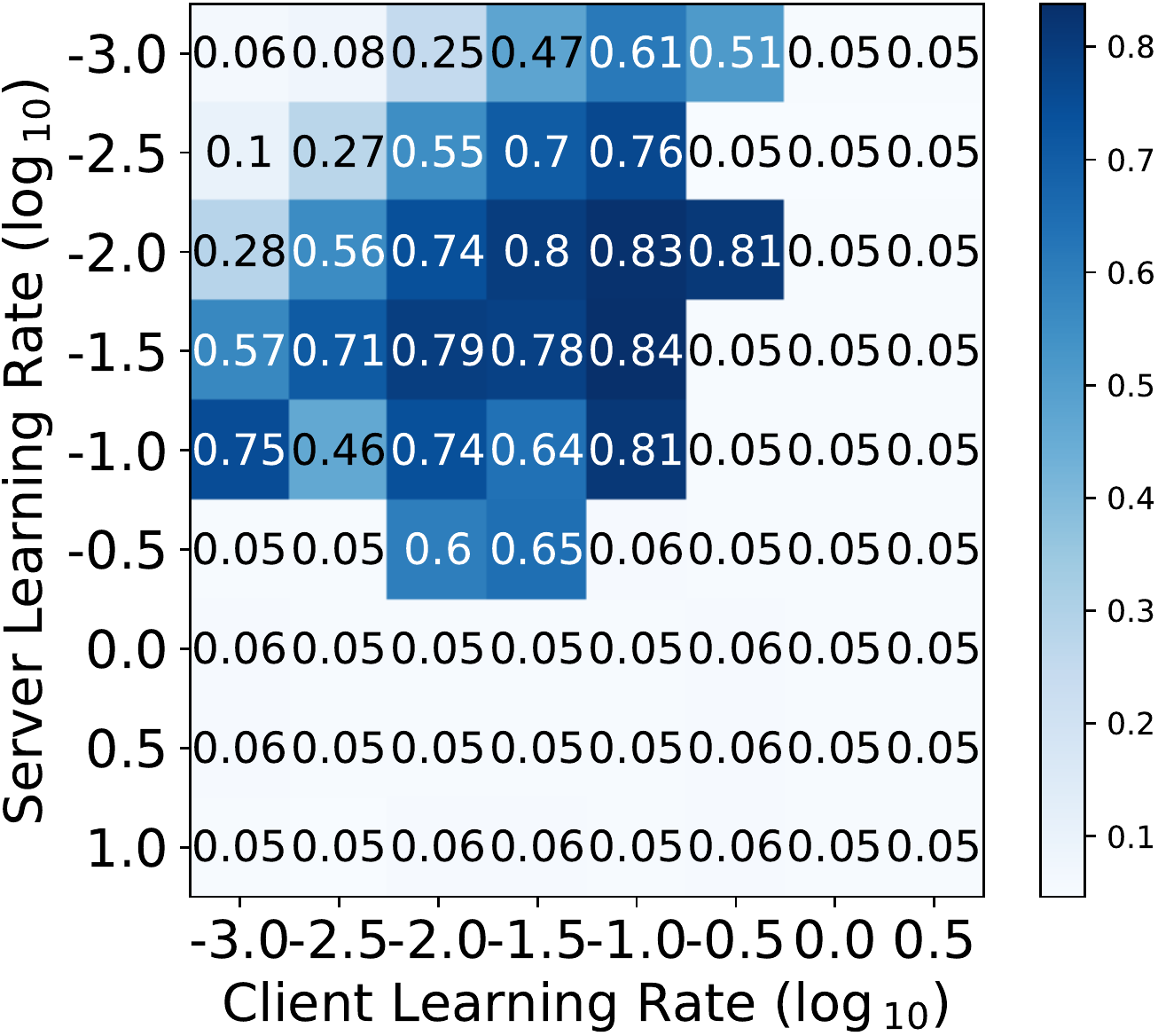}
\label{fig:heatmap-fedadagrad}} \\
\subfloat[FedAdam]{\includegraphics[width=0.24\textwidth]{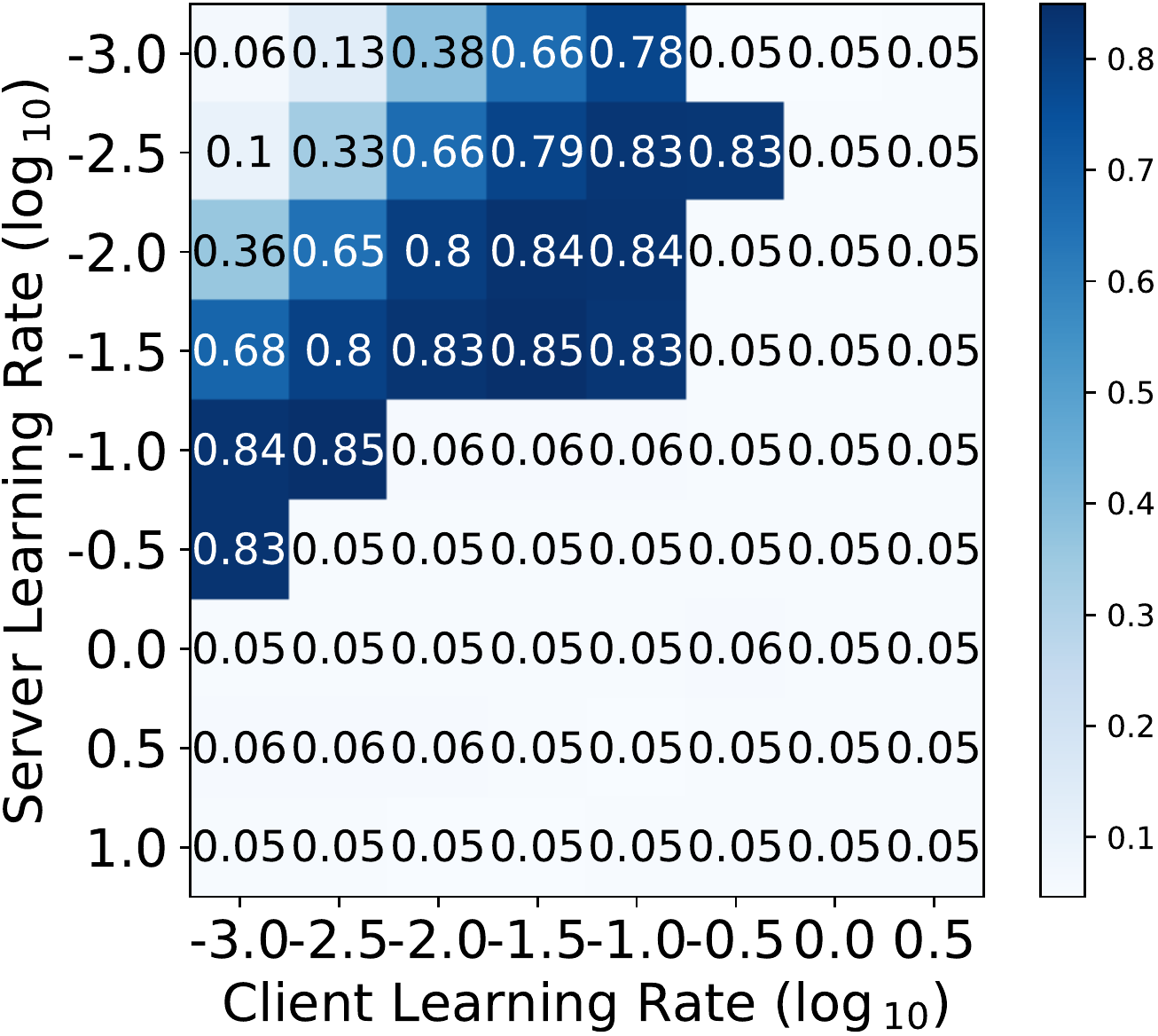}
\label{fig:heatmap-fedadam}}
\subfloat[FedYogi]{\includegraphics[width=0.24\textwidth]{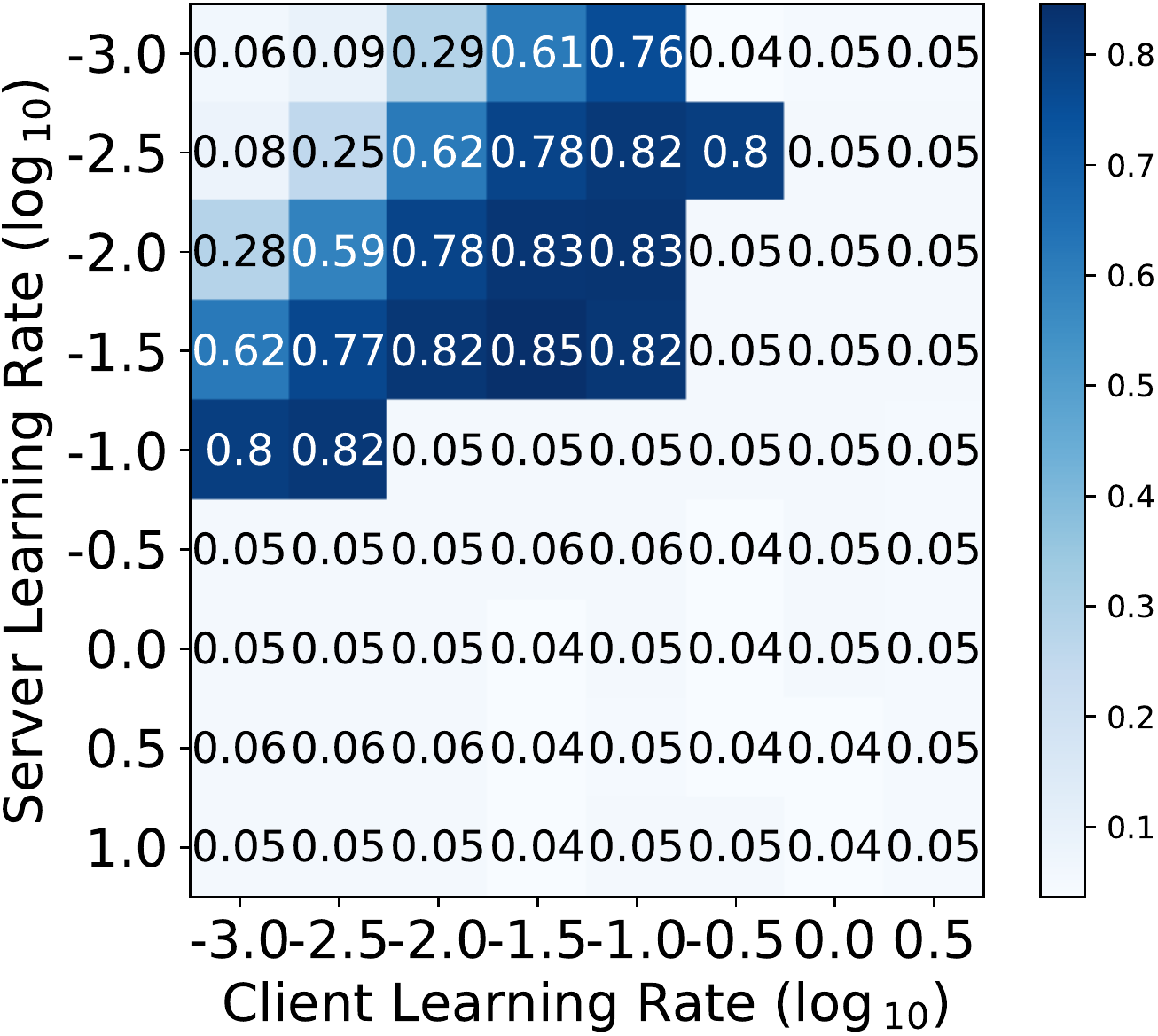}
\label{fig:heatmap-fedyogi}}
\caption{Accuracy Heatmap of (a) FedAvgM, (b) FedAdagrad, (c) FedAdam and (d) FedYogi.}
\label{fig:acc-heatmap}
\end{figure}

\textit{\NAME~vs CGAU.} CGAU uses gated activation units on top of a pre-trained model to enable client-specific expression of heterogeneous data. We train a 1-layer and a 2-layer CGAU classifier with 256 units, respectively (namely FineTunning+1$\times$CGAU and FineTunning+2$\times$CGAU). The dropout layer is not used as the authors did because we observed a 3.3\% drop in accuracy after using them. Figs. \ref{fig:acc-vs-cgau} and \ref{fig:loss-vs-cgau} show that FineTunning+1$\times$CGAU achieves a higher accuracy of 83.3\%, which improves the baseline accuracy by 1.2\% and benefits from the fast convergence speed of the pre-trained model. Despite these gains, the proposed \NAME~can still achieve 2.7\% higher accuracy, lower test loss, and faster convergence.

\textit{\NAME~vs FedMMD.} FedMMD uses transfer learning\cite{pan2009survey} to better merge the knowledge of the global model into the local model. As suggested by the authors, we use the MMD distance and the penalty coefficient $\gamma=0.1$. As shown in Figs. \ref{fig:acc-vs-fedmmd} and \ref{fig:loss-vs-fedmmd}, FedMMD improves the baseline accuracy by 0.9\% and achieves the accuracy of 83.0\%, but the proposed \NAME~further improves that by another 3\%.

\textit{\NAME~vs FedFusion.} FedFusion fuses the global and local features using operators such as $1\times1$ convolution (FedFusion+Conv), vector weighted average (FedFusion+Multi) and scalar weighted average (FedFusion+Single). However, the results in Figs. \ref{fig:acc-vs-fedfusion} and \ref{fig:loss-vs-fedfusion} show that FedFusion+Multi and FedFusion+Conv only achieve the accuracy similar to the baseline (82.0\% and 81.7\%, respectively), and FedFusion+Single even decreases the accuracy by 1.4\%. Instead, the proposed \NAME~is obviously better and faster.

\textit{\NAME~vs FedOpt.} FedOpt is a general paradigm for a series of adaptive federated optimizers, which dynamically adjusts learning rates of all gradients to accelerate convergence, including FedAvgM, FedAdagrad, FedAdam, and FedYogi. Preliminary experiments show that the convergence performance of these approaches has indeed significantly improved, but they are particularly sensitive to initial learning rates. As the authors did, we search for the best setting of the client-side and server-side initial learning rates in Fig. \ref{fig:acc-heatmap} and give the best results in Figs. \ref{fig:acc-vs-fedopt} and \ref{fig:loss-vs-fedopt}. Other hyperparameters follow the authors' setting, for example, $\beta=0.9$ for FedAvgM, $\beta_1=\beta_2=0$ for FedAdagrad, $\beta_1=0.9, \beta_2=0.99$ for FedAdam and FedYogi, and $\tau=0.001$. The results show that these approaches can improve the baseline accuracy by $1.7\%\sim2.9\%$ with fast convergence speed, especially FedAdam. However, the accuracy of the proposed \NAME~is still 1\% higher.

To sum up, FedAvgM, FedAdagrad, FedAdam, and FedYogi are generally better than other comparison approaches (some of which cannot even reach the accuracy of 82\%, marked with ``$\times$'' in Table \ref{table:comparison}). Instead, \NAME~achieves the state-of-the-art accuracy of 86.0\%, which is 3.9\% higher than the baseline and 3.5\% higher than the averaged accuracy. In addition, \NAME~can also reach the accuracy of 82\% in only 147 rounds, which is 3.3$\times$ faster than FedAvg and reduces training rounds by 59\% on average. These comprehensive experiments prove the effectiveness and efficiency of \NAME.
\section{Conclusion}
\new{FL in IIoT is emerging as a field of great value with increasing interest from both academia and industry, however, it still faces the challenge of non-i.i.d. data, which currently remains open. In this paper, we propose \NAME, which is a hierarchical cloud-edge-end FL framework for 5G empowered modern industries. To minimize the divergence in data distributions among factories, we propose a constrained gradient-based optimizer, namely \SAMPLERNAME, to select a subset of devices in each factory to construct homogeneous FL super nodes. \SAMPLERNAME~can  find a desirable selection strategy in a very short time, and can also be used for other practical cases such as game matching. Then, to eliminate the impact of residual non-i.i.d. data within the super nodes, we use a compound-step synchronization protocol to coordinate the training process. This protocol uses the data heterogeneity-insensitive one-step synchronization protocol within the super nodes to suppress the negative impact of data heterogeneity, then uses the multi-step synchronization protocol among the super nodes to reduce communication frequency.  The proposed approach takes into account the natural geographical clustering property of factory devices and can adapt to rapidly changing streaming data at runtime, without exposing confidential data in high-risk data manipulation. Theoretical analysis shows that \NAME~has both the convergence rate and optimality gap better than the benchmark FedAvg, and can be more time-efficient under a relaxed hyperparameter condition. Extensive experiments compared to ten advanced approaches demonstrate the state-of-the-art performance of \NAME~on non-i.i.d. data.}

\end{document}